\documentclass[11pt]{article}
\pdfoutput=1  
\usepackage{etoolbox}

\usepackage{booktabs}       %

\newtoggle{arxiv}
\toggletrue{arxiv}

\def\ddefloop#1{\ifx\ddefloop#1\else\ddef{#1}\expandafter\ddefloop\fi}
\def\ddef#1{\expandafter\def\csname bb#1\endcsname{\ensuremath{\mathbb{#1}}}}
\ddefloop ABCDEFGHIJKLMNOPQRSTUVWXYZ\ddefloop

\def\ddefloop#1{\ifx\ddefloop#1\else\ddef{#1}\expandafter\ddefloop\fi}
\def\ddef#1{\expandafter\def\csname frak#1\endcsname{\ensuremath{\mathfrak{#1}}}}
\ddefloop ABCDEFGHIJKLMNOPQRSTUVWXYZ\ddefloop

\def\ddefloop#1{\ifx\ddefloop#1\else\ddef{#1}\expandafter\ddefloop\fi}
\def\ddef#1{\expandafter\def\csname fr#1\endcsname{\ensuremath{\mathfrak{#1}}}}
\ddefloop ABCDEFGHIJKLMNOPQRSTUVWXYZ\ddefloop

\def\ddefloop#1{\ifx\ddefloop#1\else\ddef{#1}\expandafter\ddefloop\fi}
\def\ddef#1{\expandafter\def\csname eul#1\endcsname{\ensuremath{\EuScript{#1}}}}
\ddefloop ABCDEFGHIJKLMNOPQRSTUVWXYZ\ddefloop

\def\ddefloop#1{\ifx\ddefloop#1\else\ddef{#1}\expandafter\ddefloop\fi}
\def\ddef#1{\expandafter\def\csname scr#1\endcsname{\ensuremath{\mathscr{#1}}}}
\ddefloop ABCDEFGHIJKLMNOPQRSTUVWXYZ\ddefloop

\def\ddefloop#1{\ifx\ddefloop#1\else\ddef{#1}\expandafter\ddefloop\fi}
\def\ddef#1{\expandafter\def\csname b#1\endcsname{\ensuremath{\mathbf{#1}}}}
\ddefloop ABCDEFGHIJKLMNOPQRSTUVWXYZ\ddefloop

\def\ddefloop#1{\ifx\ddefloop#1\else\ddef{#1}\expandafter\ddefloop\fi}
\def\ddef#1{\expandafter\def\csname bhat#1\endcsname{\ensuremath{\hat{\mathbf{#1}}}}}
\ddefloop ABCDEFGHIJKLMNOPQRSTUVWXYZ\ddefloop

\def\ddefloop#1{\ifx\ddefloop#1\else\ddef{#1}\expandafter\ddefloop\fi}
\def\ddef#1{\expandafter\def\csname btil#1\endcsname{\ensuremath{\tilde{\mathbf{#1}}}}}
\ddefloop ABCDEFGHIJKLMNOPQRSTUVWXYZ\ddefloop

\def\ddefloop#1{\ifx\ddefloop#1\else\ddef{#1}\expandafter\ddefloop\fi}
\def\ddef#1{\expandafter\def\csname bst#1\endcsname{\ensuremath{\mathbf{#1}^\star}}}
\ddefloop ABCDEFGHIJKLMNOPQRSTUVWXYZ\ddefloop

\def\ddefloop#1{\ifx\ddefloop#1\else\ddef{#1}\expandafter\ddefloop\fi}
\def\ddef#1{\expandafter\def\csname bst#1\endcsname{\ensuremath{\mathbf{#1}^\star}}}
\ddefloop abcdefghijklmnopqrstuvwxyz\ddefloop

\def\ddefloop#1{\ifx\ddefloop#1\else\ddef{#1}\expandafter\ddefloop\fi}
\def\ddef#1{\expandafter\def\csname bhat#1\endcsname{\ensuremath{\hat{\mathbf{#1}}}}}
\ddefloop abcdefghijklmnopqrstuvwxyz\ddefloop

\def\ddefloop#1{\ifx\ddefloop#1\else\ddef{#1}\expandafter\ddefloop\fi}
\def\ddef#1{\expandafter\def\csname b#1\endcsname{\ensuremath{\mathbf{#1}}}}
\ddefloop abcdefghijklnopqrstuvwxyz\ddefloop

\def\ddefloop#1{\ifx\ddefloop#1\else\ddef{#1}\expandafter\ddefloop\fi}
\def\ddef#1{\expandafter\def\csname barb#1\endcsname{\ensuremath{\bar{\mathbf{#1}}}}}
\ddefloop abcdefghijklmnopqrstuvwxyz\ddefloop

\def\ddef#1{\expandafter\def\csname c#1\endcsname{\ensuremath{\mathcal{#1}}}}
\ddefloop ABCDEFGHIJKLMNOPQRSTUVWXYZ\ddefloop
\def\ddef#1{\expandafter\def\csname h#1\endcsname{\ensuremath{\widehat{#1}}}}
\ddefloop ABCDEFGHIJKLMNOPQRSTUVWXYZ\ddefloop
\def\ddef#1{\expandafter\def\csname hc#1\endcsname{\ensuremath{\widehat{\mathcal{#1}}}}}
\ddefloop ABCDEFGHIJKLMNOPQRSTUVWXYZ\ddefloop
\def\ddef#1{\expandafter\def\csname t#1\endcsname{\ensuremath{\widetilde{#1}}}}
\ddefloop ABCDEFGHIJKLMNOPQRSTUVWXYZ\ddefloop
\def\ddef#1{\expandafter\def\csname tc#1\endcsname{\ensuremath{\widetilde{\mathcal{#1}}}}}
\ddefloop ABCDEFGHIJKLMNOPQRSTUVWXYZ\ddefloop

\usepackage{boxedminipage}
\usepackage{multirow,nicefrac}
\usepackage{makecell,upgreek}
\usepackage{footnote}
\usepackage{longtable}
\usepackage[shortlabels]{enumitem}
\usepackage{tablefootnote}
\usepackage[T1]{fontenc}
\usepackage{verbatim} 
\usepackage[utf8]{inputenc}

\usepackage{booktabs}   
\usepackage{float}

\usepackage{xcolor}              
\usepackage{preamble/color-edits}
\addauthor{ms}{magenta}

\iftoggle{arxiv}
{
\usepackage{amsthm}
\usepackage{subfig}
  \usepackage[breaklinks=true]{hyperref}

  \usepackage{fullpage}
}
{
  
}

  \usepackage[noend]{algpseudocode}
  \usepackage{algorithm}

\usepackage{amsfonts,amssymb}
\usepackage{mathtools}
\usepackage[mathscr]{euscript}
\DeclareMathSymbol{\shortminus}{\mathbin}{AMSa}{"39}

	\usepackage{natbib}

\usepackage{prettyref}

\usepackage[capitalise,nameinlink]{cleveref}
\Crefname{equation}{Eq.}{Eqs.}
\Crefname{assumption}{Assumption}{Assumptions}
\Crefname{condition}{Condition}{Conditions}
\Crefname{claim}{Claim}{Claims}

\usepackage{breakcites,bm}

\hypersetup{colorlinks,citecolor=blue,linkcolor = blue}
\usepackage{latexsym}
\usepackage{relsize}
\usepackage{thm-restate}
\usepackage{appendix}

\usepackage{xcolor}
\usepackage{dsfont}

\newcommand{\R}{\mathbb{R}}

\numberwithin{equation}{section}

\def\bu{\mathbf{u}}

\def\by{\mathbf{y}}
\def\bz{\mathbf{z}}

\def\bu{\mathbf{u}}
\def\bv{\mathbf{v}}

\def\bw{\mathbf{w}}

\DeclareFontFamily{U}{mathx}{\hyphenchar\font45}
\DeclareFontShape{U}{mathx}{m}{n}{
      <5> <6> <7> <8> <9> <10>
      <10.95> <12> <14.4> <17.28> <20.74> <24.88>
      mathx10
      }{}
\DeclareSymbolFont{mathx}{U}{mathx}{m}{n}
\DeclareFontSubstitution{U}{mathx}{m}{n}
\DeclareMathAccent{\widecheck}{0}{mathx}{"71}
\DeclareMathAccent{\wideparen}{0}{mathx}{"75}

\newcommand{\ignore}[1]{}

\DeclareMathOperator{\BigOm}{\mathcal{O}}

\newcommand{\BigOh}[1]{\BigOm\left({#1}\right)}
\DeclareMathOperator{\BigOmtil}{\widetilde{\mathcal{O}}}
\newcommand{\BigOhTil}[1]{\BigOmtil\left({#1}\right)}
\DeclareMathOperator{\BigTm}{\Theta}

\DeclareMathOperator{\BigWm}{\Omega}
\newcommand{\BigThetaTil}[1]{\widetilde{\BigTm}\left({#1}\right)}
\newcommand{\BigOmega}[1]{\BigWm\left({#1}\right)}

\newcommand{\algcomment}[1]{\hfill\textcolor{blue}{\texttt{\footnotesize{\textbf{(\% #1)}}}}}

\iftoggle{arxiv}{
	\theoremstyle{plain}
	\newtheorem{theorem}{Theorem}

	\newtheorem{lemma}{Lemma}[section]

	\newtheorem{corollary}{Corollary}[section]
	\newtheorem{proposition}[lemma]{Proposition}

	\theoremstyle{definition}

	\newtheorem{definition}{Definition}[section]
	\newtheorem{example}{Example}[section]

	\newtheorem{remark}{Remark}[section]

  \newtheorem{assumption}{Assumption}[section]
  \newtheorem{condition}{Condition}[section]

\makeatletter
\newcommand{\neutralize}[1]{\expandafter\let\csname c@#1\endcsname\count@}
\makeatother

\newtheorem*{theorem*}{Theorem}
\newtheorem*{lemma*}{Lemma}
\newtheorem*{corollary*}{Corollary}
\newtheorem*{proposition*}{Proposition}
\newtheorem*{claim*}{Claim}
\newtheorem*{fact*}{Fact}
\newtheorem*{observation*}{Observation}

\newtheorem*{definition*}{Definition}
\newtheorem*{remark*}{Remark}
\newtheorem*{example*}{Example}

\newtheoremstyle{named}{}{}{\itshape}{}{\bfseries}{}{.5em}{\Cref{#3} {\normalfont (informal)} }
{}
\theoremstyle{named}

\theoremstyle{plain}

}
{

\newtheorem{assumption}{Assumption}[section]
}

\DeclareMathAlphabet{\mathbfsf}{\encodingdefault}{\sfdefault}{bx}{n}

\DeclareMathOperator*{\argmin}{arg\,min}
\DeclareMathOperator*{\argmax}{arg\,max}

\newcommand{\norm}[1]{\left|\left|#1\right|\right|}
\newcommand{\abs}[1]{\left| #1 \right|}

\newcommand{\floor}[1]{\lfloor #1 \rfloor}

\newcommand{\filt}{\mathscr{F}}
\newcommand{\Ztil}{\widetilde{Z}}
\newcommand{\kstar}{k^\star}
\newcommand{\ellbar}{\overline{\ell}}
\newcommand{\kbar}{\overline{k}}

\newcommand{\phibar}{\overline{\phi}}
\newcommand{\epsilontil}{\widetilde{\epsilon}}
\newcommand{\coeff}{\mathrm{coeff}}
\newcommand{\sigpoly}[1]{\sigma_{\mathrm{poly}, #1}}
\newcommand{\bwstar}[1]{\bw_{#1}^\star}
\newcommand{\bbaru}{\bar{\bu}}
\newcommand{\muhat}{\widehat{\mu}}
\newcommand{\dhat}{\widehat{d}}

\newcommand{\rr}{\mathbb{R}}
\newcommand{\ee}{\mathbb{E}}
\newcommand{\pp}{\mathbb{P}}
\newcommand{\elltil}{\widetilde{\ell}}
\newcommand{\inprod}[2]{\left\langle #1, #2 \right\rangle}
\newcommand{\Ltil}{\widetilde{L}}
\newcommand{\thetamax}{\theta_{\textrm{max}, i}}
\newcommand{\thetamin}{\theta_{\textrm{min},i}}
\newcommand{\sigdir}{\sigma_{\mathrm{dir}}}

\newcommand{\I}{\mathbb{I}}

\newcommand{\thetad}{\theta_{\mathrm{d}}}
\newcommand{\thetac}{\theta_{\mathrm{c}}}
\newcommand{\Thetad}{\Theta_{\mathrm{d}}}
\newcommand{\Thetac}{\Theta_{\mathrm{c}}}

\DeclareMathOperator{\sign}{sign}
\newcommand{\thetatil}{\widetilde{\theta}}

\newcommand{\bracknum}{\cN_{\cM, []}}

\DeclareMathOperator{\expo}{Exp}

\DeclareMathOperator{\reg}{Reg}

\newcommand{\algfont}[1]{\mathsf{#1}}
\newcommand{\ermoracle}{\algfont{ERMOracle}}

\usepackage{authblk}

\title{Oracle-Efficient Smoothed Online Learning for Piecewise Continuous Decision Making}

\author{Adam Block}
\author{Alexander Rakhlin}
\author{Max Simchowitz}
\affil{MIT}
\date{}

\usepackage{autonum}
\begin{document}

\maketitle

\begin{abstract}
Smoothed online learning has emerged as a popular framework to  mitigate the substantial loss in statistical and computational complexity that arises when one moves from classical to adversarial learning.  Unfortunately, for some spaces, it has been shown that efficient algorithms suffer an exponentially worse regret than that which is minimax optimal, even when the learner has access to an optimization oracle over the space.  To mitigate that exponential dependence, this work introduces a new notion of complexity, the generalized bracketing numbers, which marries constraints on the adversary to the size of the space, and shows that an instantiation of Follow-the-Perturbed-Leader can attain low regret with the number of calls to the optimization oracle scaling optimally with respect to average regret.  We then instantiate our bounds in several problems of interest, including online prediction and planning of piecewise continuous functions, which has many applications in fields as diverse as econometrics and robotics.
\end{abstract}

\tableofcontents

\section{Introduction}


The online learning setting has become the most popular regime for studying sequential decision making with dependent and potentially adversarial data. While this paradigm is attractive due to its great generality and minimal set of assumptions \citep{cesa2006prediction}, the worst-case nature of the adversary creates statistical and computational challenges \citep{rakhlin2015online,littlestone1988learning,hazan2016computational}.  In order to mitigate these difficulties, \citet{rakhlin2011online} proposed the \emph{smoothed} setting, wherein the adversary is constrained to sample data from a distribution whose likelihood ratio is bounded above by $1 / \sigma$ with respect to a fixed dominating measure, which ensures that the adversary cannot choose worst-case inputs with high probability.  As in other online learning settings, performance is measured via \emph{regret} with respect to a best-in-hindsight comparator \citep{cesa2006prediction}. 

Recent works have demonstrated strong computational-statistical tradeoffs in smoothed online learning: while there are statisticaly efficient algorithms that can enjoy regret \emph{logarithmic} in $1/\sigma$, oracle-efficient algorithms necessarily suffer regret scaling \emph{polynomially} in $1/\sigma$ \citep{haghtalab2022oracle,haghtalab2022smoothed,block2022smoothed}, where the learner is assumed access to an Empirical Risk Minimization (ERM) oracle that is able to efficiently optimize functionals on the parameter space. This gap is significant, because in many applications of interest, the natural scaling of $\sigma$ is \emph{exponential} in ambient problem dimension \citep{block2022efficient}.

A natural question remains: under which types of smoothing is it possible to design oracle-efficient algorithms with regret that scales \emph{polynomially} in problem dimension? A partial answer was provided by \cite{block2022efficient}, who demonstrate an efficient algorithm based on the John Ellipsoid  which attains $\log(T/\sigma)\cdot\mathrm{poly}(\text{dimension})$-regret for \emph{noiseless} linear classification, and for a suitable generalization to classification with polynomial features. They also demonstrate that, under a different smoothness condition - $\sigdir$-directional smoothness - the perceptron algorithm automatically provides regret sublinear-in-$T$ and polynomial in $1/\sigdir$. Crucially, $\sigdir$ is \emph{dimension-free} for many distributions of interest, circumventing the curse-of-dimension encountered in previous $\mathrm{poly}(1/\sigma)$-regret bounds \citep{block2022smoothed,haghtalab2022smoothed}.

In this work we take oracle-efficiency as a necessary precondition and expand the set of problems that efficient smoothed online learning can address.  A central example to keep in mind is that of piecewise affine (PWA) functions, where a PWA function is defined by a finite set of regions in Euclidean space, within each of which the function is affine.  Such classes naturally arise in segmented regression applications common in statistics and econometrics \citep{feder1975asymptotic,bai1998estimating,yamamoto2013estimating}, as well as in popular models for control systems \citep{borrelli2003constrained,henzinger1998hybrid}.

Unfortunately, because of the discontinuities that arise when crossing regions, PWA regressors are \emph{not} learnable in the adversarial setting even with unbounded computation time, due to the fact that they contain the class of linear thresholds, whose lack of online learnability is well-known \citep{littlestone1988learning}; however, a smoothness assumption is natural in this setting, due to the injection of noise empiricists already incorporate \citep{posa2014direct,suh2022differentiable}.  Unfortunately, the nature of the injected noise is such that the smoothness parameter $\sigma$ will be exponential in the dimension of the context space, as above, and thus previous guarantees do not suffice for applications.  We are thus left with the question of designing practical algorithms that are provably (oracle-)efficient in the smoothed online learning setting.

Below, we will propose a measure of complexity based on classical bracketing numbers \citep{blum1955convergence,gine2021mathematical} that, if bounded, leads to a practical algorithm that experiences provably small regret.  In particular, we will consider instantiations of the well-known Follow-the-Perturbed-Leader (FTPL) algorithm \citep{kalai2005efficient}, where, at each time $1 \leq t \leq T$, we sample a random path $\omega_t(\theta)$ on $\Theta$ and select $\theta_t \in \argmin_{\theta} L_{t-1}(\theta) + \omega_t(\theta)$, with $L_{t-1}(\theta)$ denoting the cumulative loss up to time $t-1$.  Standard analyses of FTPL \citep{agarwal2019learning,suggala2020online,haghtalab2022oracle, block2022smoothed} require that the loss functions be Lipschitz in the parameter $\theta$, which clearly does not hold for the central example of PWA functions.  We show, however, that smoothness guarantees that many loss functions are Lipschitz \emph{in expectation}, up to an additive constant depending on the complexity of the class as measured by our proposed generalization of bracketing numbers.  Using this fact, we provide a template for proving regret guarantees for a lazy instantiation of FTPL.

While the theory described above may be of technical interest in its own right, we instantiate our results in several examples.  We replace the standard notion of smoothness with the related concept of directional smoothness introduced above \citep{block2022efficient}.
We adapt results from \citet{agarwal2019learning,suggala2020online} on FTPL with an exponentially distributed perturbation and exhibit a practical and provably low-regret algorithm for piecewise continuous loss functions with generalized affine boundaries.  We then generalize this result to loss functions with polynomial boundaries, assuming a more constrained adversary, and finally instantiate our results in a setting motivated by robotic planning. In more detail:

\begin{itemize}
    \item In \Cref{sec:genbrackets}, we introduce a new measure of the size of a class, the generalized bracketing number, which combines assumptions on the adversary with the complexity of the space and thus can be small in many situations of interes.  We use generalized bracketing numbers to prove Proposition \ref{prop:lazyftpl}, which says that if an adversary is suitably constrained and the generalized bracketing number with respect to a particular pseudo-metric is controlled, then a lazy version of FTPL experiences low regret.  Along the way, we show in Proposition \ref{prop:singlestepbracketing} that control of the generalized bracketing number leads to a concentration inequality that is uniform over both parameters and adversaries.
    \item In \Cref{thm:lazyftplexponential}, we apply the general theory developed in \Cref{sec:genbrackets} to the special case of finite dimensional $\Theta$.  In particular, by adapting arguments of \citet{agarwal2019learning,suggala2020online}, we show that if the generalized bracketing numbers of $\Theta$ are controlled, then Algorithm \ref{alg:lazyftplexponential} can achieve average regret at most $\epsilon$ with the optimal $\BigOhTil{\epsilon^{-2}}$ number of calls to the ERM oracle.
    \item In \Cref{thm:pwabracketing} and Corollary \ref{cor:pwaftpl}, we consider an even more concrete setting, where the loss function is piecewise continuous with affine boundaries.  In particular, we show that if the adversary is $\sigdir$-directionally smooth, then Algorithm \ref{alg:lazyftplexponential} attains average regret $\epsilon$ with only $\BigOhTil{\sigdir^{-1} \epsilon^{-2}}$ calls to the ERM oracle, removing the exponential dependence on the dimension that would come from applying \citet{block2022smoothed} and attaining optimal dependence on $\epsilon$.
    \item In \Cref{thm:polyregret}, we generalize the results of Corollary \ref{cor:pwaftpl} and show that if the adversary is further constrained to be polynomially smooth (see Definition \ref{def:polysmooth}) and the loss function is piecewise continuous with boundaries defined by polynomials of degree at most $r$, then Algorithm \ref{alg:lazyftplexponential} can achieve average regret $\epsilon$ with at most $\BigOhTil{\epsilon^{-2r}}$ calls to the ERM oracle.
    \item In \Cref{sec:planning}, we consider a setting of piecewise Lipschitz ``hybrid'' dynamical systems \citep{henzinger1998hybrid}, where the boundaries within regions are either linear are polynomial. These can model a number of dynamical systems popular in robotics, notably piecewise affine systems \citep{borrelli2003constrained,marcucci2019mixed} and piecwise-polynomial systems 
\citep{posa2015stability}. We  demonstrate in \Cref{thm:planning} that, with smoothning in the inputs and dynamics, our proposed FTPL algorithm attains low-regret in an online planning setting. To our knowledge, this is the first low-regret algorithm for planning in hybrid systems that exhibit discontinuities. 
\end{itemize}
We begin the paper by formally setting up the problem and introducing a number of prerequisite notions, before continuing to state and discuss our results. An extended discussion of related work is deferred to \Cref{app:relatedwork} for the sake of space.


\section{Formal Setting and Notation}

Formally, we consider the problem of online learning with a constrained adversary.  Given some decision space $\Theta$ and context space $\cZ$, as well as a loss function $\ell: \Theta \times \cZ \to [0,1]$, online learning proceeds in rounds $1 \leq t \leq T$.  At each time $t$, the adversary selects some $z_t \in \cZ$ and the learner selects some $\theta_t \in \Theta$ and suffers loss $\ell(\theta_t, z_t)$ with the goal of minimizing regret with respect to the best $\theta \in \Theta$ in hindsight, $\ee\left[\reg_T\right] = \ee\left[\sum_{t = 1}^T \ell(\theta_t, z_t) - \inf_{\theta \in \Theta} \sum_{t = 1}^T \ell(\theta, z_t)\right]$.  For the purposes of measuring oracle complexity, we will be particularly interested in the normalized regret $T^{-1} \reg_T$.  Frequently in applications, we will consider the special case of online supervised learning where $\cZ = \cX \times \cY$ and $z = (x,y)$ consists of a context $x$ and label $y$; in this case, we distinguish between \emph{proper} learning, where the learner chooses $\theta_t$ before seeing $x_t$, and \emph{improper} learning, where the learner is able to choose $\theta_t$ depending on the revealed $x_t$.

Due to the statistical and computational challenges of fully adversarial online learning \citep{rakhlin2015online,hazan2016computational}, we will constrain the adversary to choose $z_t \sim p_t$, where $p_t \in \cM \subset \Delta(\cZ)$ is a distribution on $\cZ$ possibly depending on the history up to time $t$ and $\cM$ is some restricted class of distributions.  In this work, we will mostly focus on the setting where $\cM$ consists of smooth distributions in some sense:
\begin{definition}\label{defn:smooth}
    Given a space $\cX$, a measure $\mu \in \Delta(\cX)$, and some $\sigma < 0$, we say that a measure $p_t$ is $\sigma$-smooth with respect to $\mu$ if the likelihood ratio with respect to $\mu$ is uniformly bounded by $\sigma^{-1}$, i.e., $\norm{\frac{d p_t}{d \mu}}_{\infty} \leq \frac 1\sigma$.  If $\cZ \subset \rr^d$ for some $d$, we say that $p_t$ is $\sigdir$-directionally smooth if, for any unit vector $\bw \in \cS^{d-1}$, the distribution of $\inprod{\bw}{\bx}$ is $\sigdir$-smooth with respect to the Lebesgue measure on the real line, where $\bx \sim p_t$.
\end{definition}
As discussed further in the related work section, smoothness has recently become a popular assumption for smoothed online learning.  Directional smoothness, introduced in \citet{block2022efficient} and used in \citet{block2023smoothed}, has provided a natural way to mitigate the dimensional dependence of standard smoothness in some commonly used systems.

Our algorithms will employ the computational primitive of an Empirical Risk Minimization (ERM) oracle:
\begin{definition}\label{def:ermoracle}
    Given a space $\Theta$, and functionals $\ell_i : \Theta \to \rr$ for $1 \leq i \leq m$, define an Empirical Risk Minimization (ERM) oracle as any oracle that optimizes over $\Theta$, i.e., $\thetatil = \ermoracle\left( \sum_{i = 1}^m \ell_i(\theta) \right) $ if $ \thetatil \in \argmin_{\theta \in \Theta} \sum_{i  =1}^m \ell_i(\theta)$.
\end{definition}

Definition \ref{def:ermoracle} is a common assumption in the study of computationally efficient online learning \citep{hazan2016computational,block2022smoothed,haghtalab2022oracle}, with many heuristics for popular function classes available for practical application \citep{lecun2015deep,garulli2012survey}.  In the sequel, we will always suppose that ther learner has access to an ERM Oracle and measure the computational complexity of the algorithm by the number of calls to $\ermoracle$.  In particular, we are interested in the oracle complexity of achieving average regret $\epsilon$, i.e., the number of oracle calls that suffice to ensure that $T^{-1} \cdot \ee\left[ \reg_T \right] \leq \epsilon$.  While in the main body we assume that $\ermoracle$ is exact for the sake of clean presentation, in the appendix we provide statements and proofs requiring only an approximate oracle, with a possibly perturbation-dependent error contributing additively to our final regret guarantees.

In the following section, we will introduce a new notion of complexity, the generalized bracketing number of a space $\Theta$.  Here, we will recall the classical notion of bracketing entropy, both for the sake of comparison and for future reference with respect to one of our results:
\begin{definition}[From Section 3.5.2 in \citet{gine2021mathematical}]\label{def:classicalbracketing}
    For a function class $\cF: \cZ \to \rr$ and a measure $\mu \in \Delta(\cZ)$, we say that a partition $\cN = \left\{ \cB_i \right\}$ of $\cF$ is an $\epsilon$-bracket with respect to $\mu$ if for all $\cB_i$, it holds that $\ee_\nu\left[ \sup_{f, g \in \cB_i} \abs{f(z) - g(z)} \right] \leq \epsilon$.  The bracketing number, $\cN_{[]}\left( \cF, \mu, \epsilon \right)$ is the minimal size of such a partition.
\end{definition}
Control of the bracketing numbers of a function class classically lead to uniform laws of large numbers and uniform central limit theorems, with many common function classes having well-behaved such numbers; for more detail, see \citep{gine2021mathematical}.

\paragraph{Notation} In the sequel, we will reserve $z$ for contexts and $\theta$ for parameters.  We will always denote the horizon by $T$, loss functions by $\ell$, and will make vectors bold.  We will use $\BigOh{\cdot}$ notation to suppress universal constants and $\BigOhTil{\cdot}$ to suppress polylogarithmic factors.  We will let $\norm{\cdot}_1$ denote the $\ell_1$ norm in Euclidean space and the unadorned $\norm{\cdot}$ denote the Euclidean norm.


\section{Follow the Perturbed Leader and Generalized Brackets}\label{sec:genbrackets}
In this section, we propose our algorithm and define the complexity parameters that ensure we experience low expected regret.  In the following section, we will provide examples.  We will consider an instantiation of the Follow-the-Perturbed-Leader (FTPL) class of algorithms \citep{kalai2005efficient}, where, at each time $1 \leq t \leq T$, we construct a sample path $\omega_t(\theta)$ drawn independently and identically across $t$ from some stochastic process on $\Theta$ and select
\begin{align}\label{eq:ftplgeneral}
    \theta_t = \argmin_{\theta \in \Theta} L_{t-1}(\theta) + \omega_t(\theta),
\end{align}
where $L_{t-1}(\theta) = \sum_{s = 1}^{t-1} \ell(\theta, z_s)$.  The classical analysis of FTPL uses the so-called `Be-The-Leader' lemma \citep[Lemma 3.1]{kalai2005efficient} to decompose regret into the size of the perturbation and the stability of the predictions, i.e., if the learner plays $\theta_t$ from \eqref{eq:ftplgeneral}, then regret is bounded as follows:
\begin{align}\label{eq:btl}
    \ee\left[ \reg_T \right] \leq 2\cdot \ee\left[ \sup_{\theta \in \Theta} \omega_1(\theta) \right] + \sum_{t = 1}^T \ee\left[ \ell(\theta_t, z_t) - \ell(\theta_{t+1}, z_t) \right].
\end{align}
Typically, the challenge in analysing the regret incurred by FTPL is in bounding the second term in \eqref{eq:btl}, the stability term.  A common assumption involved in this analysis is that the loss $\ell$ is Lipschitz in $\theta$ \citep{agarwal2019learning,suggala2020online,block2022smoothed}; unfortunately, for many classes of interest, this assumption does not hold.

To motivate our approach, consider the simple setting of learning linear thresholds, where $\theta \in [0,1]$ and $\ell(\theta, z) = \I\left[y \neq \sign(x - \theta) \right]$ for $z = (x,y) \in \cZ = [0,1] \times \left\{ \pm 1 \right\}$.  In this case, it is clear that $\theta \mapsto \ell(\theta, z)$ is not Lipschitz (or even continuous) and so the results of \citet{agarwal2019learning,suggala2020online} do not apply; however, a simple computation tells us that if the adversary is $\sigma$-smooth with respect to the Lebesgue measure, then $\theta \mapsto \ee_{z}\left[ \ell(\theta, z) \right]$ \emph{is} Lipschitz.  Na{\"i}vely, we might then hope that the stability term $\ee\left[ \ell(\theta_t, z_t) - \ell(\theta_{t+1}, z_t) \right]$ can be controlled by $\abs{\theta_t - \theta_{t+1}}$ and a similar argument as in \citet{agarwal2019learning,suggala2020online} could be applied.  This idea does not work because, while it is true that for any fixed $\theta \in \Theta$, smoothness of $z_t$ conditioned on the history implies that $\ee\left[ \ell(\theta_t, z_t) - \ell(\theta, z_t) \right] \lesssim \abs{\theta_t - \theta}$, in fact $\theta_{t+1}$ depends on $z_t$ and so it is \emph{not} true that the distribution of $z_t$ conditioned on $\theta_{t+1}$ is necessarily smooth.  We will not wholly discard the approach, however; instead, we will show that if the class of functions $\theta \mapsto \ell(\theta, z)$ is small with respect to a particular notion of complexity, then a similar argument holds.  To make this precise, consider the following definition:
\begin{definition}\label{def:genbrackets}
    Let $\cM$ be a class of distributions on some space $\cZ$ and suppose that $\rho : \Theta \times \Theta \times \cZ \to \rr$ is a pseudo-metric on the space $\Theta$, parameterized by elements of $\cZ$.  We say that a set $\left\{ (\theta_i, \cB_i) \right\} \subset \Theta \times 2^{\Theta}$ is a generalized $\epsilon$-bracket if $\Theta \subset \bigcup_i \cB_i$ and for all $i$, it holds that
    \begin{align}
        \sup_{\nu \in \cM} \ee_{z \sim \nu}\left[ \sup_{\theta \in \cB_i} \rho(\theta, \theta_i, z) \right] \leq \epsilon.
    \end{align}
    We denote by $\bracknum\left( \Theta, \rho, \epsilon \right)$ the minimal size of a generalized $\epsilon$-bracket.
\end{definition}
Note the similarity of Definition \ref{def:genbrackets} with the classical notion from Definition \ref{def:classicalbracketing}: generalized brackets require that the expected diameter of a given partition $\cB_i$ is small \emph{uniformly over measures} in some class $\cM$; in fact, if $\cM$ is a singleton, we recover the classical notion.    The utility of generalized $\epsilon$-brackets over other notions of complexity, like standard covering numbers is as follows:
\begin{proposition}\label{prop:singlestepbracketing}
    Let $\cM$ and $\rho$ be as in Definition \ref{def:genbrackets} and suppose that $z_1, \dots, z_n$ are generated such that the law $p_i$ of $z_i$ conditioned on $\sigma$-algebra $\filt_i$ generated by the $z_j$ up to time $i$ satisfies $p_i \in \cM$ for all $1 \leq i \leq n$.  Suppose further that for all $z \in \cZ$, it holds that $\sup_{\theta, \theta' \in \Theta} \rho(\theta,\theta', z) \leq D$.  Then for any $\epsilon,\delta> 0$, with probability at least $1 - \delta$, it holds simultaneously for all $\theta, \theta' \in \Theta$ that:
    \begin{align}\label{eq:singlestepbracketing}
        \abs{\sum_{i = 1}^n \rho(\theta, \theta', z_i)} \leq 4 n \cdot \sup_{\nu \in \cM} \ee_\nu \left[ \rho(\theta, \theta', z) \right] + 8 \epsilon n + 6 D^2 \log\left( \frac{2 \bracknum\left( \Theta, \rho, \epsilon \right)}{\delta} \right).
    \end{align}
\end{proposition}
The proof of Proposition \ref{prop:singlestepbracketing} can be found in \Cref{app:keybound} and proceeds by applying Freedman's inequality and controlling the supremum of a sum by the sum of suprema.  It is somewhat surprising that, despite this seemingly very loose bound, we are able to achieve below the expected $\BigOhTil{\epsilon^{-2}}$ oracle complexity guarantees in a wide variety of settings.

Critically, because \eqref{eq:singlestepbracketing} holds uniformly over $\theta' \in \Theta$, we may apply Proposition \ref{prop:singlestepbracketing} to $\theta' = \theta_{t+1}$ and escape the challenge presented by $z_t$ not being smooth when conditioned on $\theta_{t+1}$.  There are two remaining problems before we can present our algorithm.  First, due to the additive statistical error in \eqref{eq:singlestepbracketing}, if $n$ is too small, then Proposition \ref{prop:singlestepbracketing} is vacuous.  To mitigate this problem, we will run FTPL in epochs.  For some fixed $n \in \bbN$, and for all $\tau \geq 1$, let $\Ltil_\tau(\theta) = L_{\tau n}(\theta)$, and define $\cI_\tau = \left\{ i | (\tau - 1)n + 1 \leq i \leq \tau n \right\}$ as well as $\elltil_\tau(\theta) = \sum_{t \in \cI_\tau} \ell(\theta, z_t)$.  We will run a lazy version of FTPL, where we update $\theta_t = \thetatil_\tau$ at the beginning of each $\cI_\tau$ and let $\theta_t = \thetatil_\tau$ until the next change of epoch.  The laziness allows the first term in \eqref{eq:singlestepbracketing} to dominate when we apply Proposition \ref{prop:singlestepbracketing}.  The full algorithm is summarized in Algorithm \ref{alg:lazyftpl}.

\begin{algorithm}[!t]
    \begin{algorithmic}[1]
    \State{}\textbf{Initialize } ERM Oracle $\ermoracle$, epoch length $n$, perturbation distribution $\Omega$
    \For{$\tau = 1,2,\dots, T / n$}
        \State{} \textbf{Sample} $\omega_\tau: \Theta \to \rr$ from $\Omega$ \algcomment{Sample Perturbation}
        \State{} $\thetatil_\tau \gets \ermoracle\left(\Ltil_\tau(\theta) + \omega_\tau(\theta) \right)$ \algcomment{Call $\ermoracle$ on perturbed losses}
        \For{$t = (\tau - 1)n + 1, \dots, \tau n$}
            \State{} \textbf{Observe} $z_t$, \textbf{Predict} $\thetatil_\tau$, \textbf{Receive} $\ell(\thetatil_\tau, z_t)$ 
        \EndFor
    \EndFor
    \end{algorithmic}
      \caption{Lazy FTPL}
      \label{alg:lazyftpl}
\end{algorithm}

The second challenge is to relate the stability terms in \eqref{eq:btl} to the pseudo-metric $\rho$ evaluated on successive $\thetatil_\tau$.  Thus, we will require that the losses satisfy the following structural condition:
\begin{definition}\label{def:pseudoisometry}
    Suppose that that $\Theta$ is a subset of some normed space equipped with norm $\norm{\cdot}$.  We say that the pseudo-metric $\rho: \Theta \times \Theta \times \cZ \to \rr$ satisfies the pseudo-isometry property with parameters $(\alpha, \beta)$ with respect to the class of distributions $\cM$ and the norm $\norm{\cdot}$ if for all $\theta, \theta' \in \Theta$, it holds that
    \begin{align}
        \sup_{\nu \in \cM} \ee_{z \sim \nu}\left[ \rho(\theta, \theta', z) \right] \leq \alpha \cdot \norm{\theta - \theta'}^\beta.
    \end{align}
\end{definition}
We are now prepared to state our first result bounding the regret of an instance of Algorithm \ref{alg:lazyftpl}:
\begin{proposition}\label{prop:lazyftpl}
    Suppose that we are in the constrained online learning setting, where the adversary is constrained to sample $z_t$ from some distribution in the class $\cM$.  Suppose further that there is a pseudo-metric $\rho$ on $\Theta$ parameterized by $\cZ$ satisfying the pseudo-isometry property of Definition \ref{def:pseudoisometry}, and for all $\theta, \theta' \in \Theta$ it holds that $\sup_{\nu \in \cM} \ee_\nu\left[ \ell(\theta, z) - \ell(\theta',z) \right] \leq \sup_{\nu \in \cM} \ee_\nu\left[ \rho(\theta, \theta', z) \right]$.  If the learner plays Algorithm \ref{alg:lazyftpl} and $\sup_{\theta, \theta' \in \Theta} \rho(\theta, \theta', z) \leq D$, then for any $\epsilon > 0$, the expected regret is upper bounded by:
    \begin{align}
        \BigOh{ \ee\left[ \sup_{\theta \in \Theta} \omega_1(\theta) \right]  +  \epsilon T + \frac{TD^2}{n} \cdot \log\left( T \cdot \bracknum\left( \Theta, \rho, \epsilon \right) \right) + 2n \alpha \cdot \sum_{\tau = 1}^{T/n} \ee\left[ \norm{\thetatil_\tau - \thetatil_{\tau + 1}}^\beta \right]}.
    \end{align}
\end{proposition}
We provide a complete proof in \Cref{app:lazyftplmaster}.  We first prove a variant of the Be-the-Leader lemma from \citet{kalai2005efficient} that allows for lazy updates, before applying Proposition \ref{prop:singlestepbracketing} along with the pseudo-isometry property to control the stability term of the lazy updates with respect to the evaluated loss functions by the stability of the learner's predictions with respect to the relevant norm.  Putting everything together concludes the proof.  We remark that, as presented, it might appear that there is no disadvantage to setting $n$ as large as possible; indeed the $n$ dependence in the final sum appears to cancel out and increasing $n$ decreases the third term.  Unsurprisingly, this is not the case as increasing $n$ reduces the stability of the learner's predictions and thus implicitly increases the final term, as is clear in the applications of this result.

Proposition \ref{prop:lazyftpl} provides a template for proving regret bounds for different instantiations of Algorithm \ref{alg:lazyftpl}.  In particular, for a given loss function $\ell(\cdot, \cdot)$, it suffices to find a pseudo-metric $\rho$, norm $\norm{\cdot}$, and noise distribution $\Omega$ such that (a) $\rho$ is a pseudo-isometry with respect to the norm $\norm{\cdot}$, (b) the generalized bracketing numbers of $\Theta$ are small with respect to $\rho$, and (c) the perturbation causes the lazy updates to be stable in the sense that $\ee\left[ \norm{\thetatil_\tau - \thetatil_{\tau+1}} \right]$ is small.  As an easy warmup for the results in the next section, we show that we can recover a weak version of the oracle-complexity upper bound of proper, smoothed online learning with the Gaussian process perturbation from \citet{block2022smoothed}, using a substantially simpler proof when the relevant function class has small bracketing entropy in the classical sense.

In this motivating example, we suppose that $\Theta = \cF$ denotes a function class and that we are in the online supervised learning setting, i.e., $\cZ = \cX \times \cY$ with $\ell(\theta, z) = \elltil(f(x), y)$.  We further suppose that the adversary is $\sigma$-smooth with respect to a known base measure $\mu$ (recall Definition~\ref{defn:smooth}).  As in \citet[Theorem 10]{block2022smoothed}, we consider a Gaussian process perturbation, where we draw $x_1, \dots, x_m \sim \mu$ independently, $\gamma_1, \dots, \gamma_m$ standard gaussians, and let $\omega(f) = \eta \cdot \sum_{i  =1}^m \gamma_i f(x_i)$.
\begin{corollary}\label{cor:gaussianperturbation}
    Suppose that we are in the smoothed online learning setting with a function class $\cF: \cX \to \left\{ \pm 1 \right\}$ and with $\elltil$ in the unit interval and Lipschitz with respect to the first argument for all choices of the second argument.  If the learner plays Algorithm \ref{alg:lazyftpl} with the Gaussian perturbation described above, then with the correct choice of hyperparameters, given in \Cref{app:gaussian}, the learner can achieve average regret $\epsilon$ with $\BigOhTil{\frac{\epsilon^{-4} L^{3/5}}{\sigma^{2/5}} \cdot \log^{3/5}\left( \cN_{[]}\left( \cF, \mu, \frac{\sigma}{L T} \right)\right)}$ calls to the ERM oracle.
\end{corollary}
Note that the oracle complexity guarantee is weaker than that of \citet{block2022smoothed}; we include Corollary \ref{cor:gaussianperturbation}, and its proof in \Cref{app:gaussian}, merely as a simple demonstration of our techniques and how they relate to more classical notions of function class complexity.  We now proceed to examples where our machinery provides novel regret bounds in fundamental settings.

\section{Exponential Perturbations and Piecewise Continuous Functions}\label{sec:examples}

\begin{algorithm}[!t]
    \begin{algorithmic}[1]
    \State{}\textbf{Initialize } ERM Oracle $\ermoracle$, epoch length $n$, perturbation size $\eta$
    \For{$\tau = 1,2,\dots, T / n$}
        \State{} \textbf{Sample} $\xi = (\xi_1, \dots, \xi_d) \stackrel{iid}{\sim} \expo(1)$ \algcomment{Sample Perturbation}
        \State{} $\thetatil_\tau \gets \ermoracle\left(\Ltil_\tau(\theta) - \eta \inprod{\xi}{\theta} \right)$ \algcomment{Call $\ermoracle$ on perturbed losses}
        \For{$t = (\tau - 1)n + 1, \dots, \tau n$}
            \State{} \textbf{Observe} $z_t$, \textbf{Predict} $\thetatil_\tau$, \textbf{Receive} $\ell(\thetatil_\tau, z_t)$ 
        \EndFor
    \EndFor
    \end{algorithmic}
      \caption{Lazy FTPL with Exponential Noise}
      \label{alg:lazyftplexponential}
\end{algorithm}

In the previous section, we observed that Proposition \ref{prop:lazyftpl} provided a template for proving regret bounds for different instantiations of FTPL and applied this technique to recover earlier results from smoothed online learning.  In this section, we provide new results for an important setting: piecewise continuous functions.  Before we formally define piecewise continuous functions, we consider the more general case where the set $\Theta \subset \rr^d$ for some dimension $d$.  The template provided by Proposition \ref{prop:lazyftpl} requires that we specify a perturbation distribution; whereas before we used a Gaussian process, here we adopt the approach of \citet{agarwal2019learning,suggala2020online} and use an exponential perturbation.  Summarized in Algorithm \ref{alg:lazyftplexponential}, we keep the lazy updating from Algorithm \ref{alg:lazyftpl} but specify $\omega(\theta) = - \eta \cdot \inprod{\xi}{\theta}$ for some scale parameter $\eta > 0$ and $\xi = (\xi_1, \dots, \xi_d)$ for $\xi_i \sim \expo(1)$ independently.  With the exponential perturbation, we have the following regret bound:
\begin{theorem}\label{thm:lazyftplexponential}
    Suppose that we are in the constrained online learning setting of Proposition \ref{prop:lazyftpl} with $\Theta \subset \rr^d$ such that $\sup_{\theta, \theta' \in \Theta} \norm{\theta - \theta'}_1 = D < \infty$.  Suppose further that the $\cZ$-parameterized pseudo-metric $\rho$ satisfies the pseudo-isometry property  of Definition \ref{def:pseudoisometry} with respect to $\ell_1$ on $\rr^d$ and that $\sup_{\nu \in \cM} \ee_\nu\left[ \ell(\theta, z) - \ell(\theta',z) \right] \leq \sup_{\nu \in \cM} \ee_\nu\left[ \rho(\theta, \theta', z) \right]$.  If the learner plays Algorithm \ref{alg:lazyftplexponential} and $\eta = \Omega(n^2)$, then the expected regret is bounded:
    \begin{align}
        \ee\left[ \reg_T \right] \leq \BigOh{ \eta + \frac{T}{n} \cdot \log\left( \bracknum(\Theta, \rho, 1/T) \right) + T \alpha \left( \frac{\log \bracknum(\Theta, \rho, 1/T)}{\eta} \right)^{\frac{\beta}{4 - 2\beta}} }.
    \end{align}
\end{theorem}
Tuning $\eta$ and $n$, regret scales as $\BigOhTil{T^{\frac{4 - 2 \beta}{4 - \beta}}}$ with $\BigOhTil{ T^{\frac{2}{4 - \beta}}}$ calls to the optimization oracle and thus $\BigOhTil{\epsilon^{-2/\beta}}$ calls to $\ermoracle$ suffice to attain average regret $\epsilon$.  In particular, in the best case, when $\beta = 1$, we recover the optimal $\BigOhTil{\epsilon^{-2}}$ oracle-complexity of attaining average regret bounded by $\epsilon$ that would arise if we called the oracle once per round and achieved regret $\BigOhTil{\sqrt{T}}$.

While a complete proof of \Cref{thm:lazyftplexponential} can be found in \Cref{app:ftpl}, we provide a brief sketch here.  Though we follow the general template of Proposition \ref{prop:lazyftpl}, we do not directly apply the result in order to get a slightly improved rate.  As in the proof of the more general proposition, we appeal to the Be-the-Leader lemma to reduce the analysis to bounding the stability of the learner's predictions with respect to the evaluated loss functions.  We then apply techniques from \citet{agarwal2019learning,suggala2020online} to show that if the stability term is small, then the learner's predictions are stable with respect to $\norm{\cdot}_1$ in $\rr^d$.  Finally, we use the pseudo-isometry assumption and control of the generalized bracketing numbers along with Proposition \ref{prop:singlestepbracketing} to conclude with a self-bounding argument.

\subsection{Piecewise-Continuous Prediction}
We now instantiate the previous result on several problems of interest.  For the rest of this section, we show that piecewise continuous functions with well-behaved boundaries allow for both small bracketing numbers and the pseudo-isometry property for properly chosen $\rho$, assuming only directional smoothness of the adversary.  Formally, we suppose that $\Theta = \Thetac \times \Thetad$ can be decomposed into continuous and discrete parts with $\Theta \subset \rr^m$ for some dimension $m$.  We construct a function $g$ as follows.  First, consider classes $g_k: \Thetac \times \cZ \to \rr$ for $1 \leq k \leq K$ such that for all $z \in \cZ$, $g_k(\cdot, z)$ is Lipschitz as a function of $\thetac$ with respect to the $\ell_1$ norm on $\Theta$.  Now, for a fixed $\phi: \Thetad \times [K] \times \cZ \to \rr$, we define
\begin{align} \label{eq:piecewisecontinuous}
    k_\phi(\thetad, z) = \argmax_{k \in [K]} \phi(\thetad, k, z), &&
    \ell(\theta, z) = g_{k_\phi(\thetad, z)}(\thetac, z).
\end{align}

While the formulation of \eqref{eq:piecewisecontinuous} combines versatility and simplicity, a related construction turns out to be easier to analyze: let $\phibar : \Thetad \times [K]^{\times 2} \times \cZ \to \rr$ such that $\phibar(\thetad, k, k', z) = -\phibar(\thetad, k', k, z)$ for all $\thetad \in \Thetad$, $k,k' \in [K]$, and $z \in \cZ$.  Further, let
\begin{align}
    \kbar_\phi(\thetad, z) = \argmax_{k \in [K]} \sum_{k' \neq k} \I\left[ \phibar(\thetad, k, k', z) \geq \phi(\thetad, k', k, z) \right],
\end{align}
with ties broken lexicagraphically, i.e., $\kbar_{\phibar}$ is the smallest index $k$ that wins the most matches of a tournament, where victory is determined by the sign of $\phibar(\thetad, k, k', z)$.  We then define
\begin{align}\label{eq:piecewisecontinuoustournaments}
    \ellbar(\theta,z) = g_{\kbar_{\phibar}(\thetad,z)}(\thetac, z). 
\end{align}
In this section, we will focus on the tournament formulation of \eqref{eq:piecewisecontinuoustournaments} for the sake of simplicity.  In \Cref{app:pwabracketing}, we will extend our results to the case of \eqref{eq:piecewisecontinuous} with an additional margin assumption.  We further remark that \eqref{eq:piecewisecontinuoustournaments} can be regarded as an improper relaxation of the natural function class in \eqref{eq:piecewisecontinuous} and thus suffices for improper online learning\footnote{See \citet{block2022smoothed} for a discussion on the difference between proper and improper online learning.}.  Finally, we note that, while we have described a tournament-style aggregation system for the sake of simplicity, as can be seen from our proof, any aggregation of the $\binom{K}{2}$ events $\phibar(\thetad, k, k', z) \geq 0$ will result in a similar statement, resulting in much greater generality.  This generalization allows, for example, to efficiently represent polytopic regions with $K$ proportional to the number of faces.

\subsection{Piecewise Continuous Prediction with Generalized Affine Boundaries}
We begin our study with the important special case of affine decision boundaries.  and note that the setting described by \eqref{eq:piecewisecontinuous} encompasses the central example of PWA functions: by letting $\Thetac = \left( \rr^{m \times d} \right)^{\times K}$, $\Thetad = (\rr^{d+1})^{\times K}$, $\cZ = \rr^d \times \rr^m$, and $\phi(\thetad, k, z) = \inprod{\bw_k}{(\bx,1)}$, we may take
\begin{align}
    \ell(\theta, z) = \norm{\by - \bW_{\kstar} \bx}^2, && \kstar = \argmax_{k \in [K]} \inprod{\bw_k}{(\bx, 1)},
\end{align}
where we add an extra coordinate of $1$ at the end to account for a possible affine constant.  We show that if $\ellbar$ is piecewise continuous as in \eqref{eq:piecewisecontinuoustournaments} with affine boundaries, then the generalized bracketing numbers are small and pseudo-isometry holds with respect to the $\ell_1$ norm as long as the adversary is $\sigdir$-directionally smooth.
\begin{theorem}\label{thm:pwabracketing}
    Suppose that $\cZ \subset \rr^d$ and that $\Theta$ is a subset of Euclidean space of $\ell_1$ diameter bounded by $D$, with $\Thetad \subset (\cS^{d})^{\times \binom{K}{2}}$; denote by $\bw_{kk'}$ the coordinates of a given $\thetad \in \Thetad$.  Suppose further that $\phibar(\thetad, k, k', \bz) = \psi(\inprod{\bw_{kk'}}{(\bz, 1)})$ for some differentiable, odd, link function $\psi: \rr \to \rr$ satisfying $a \leq \abs{\psi'(x)} \leq A$ for all $x$, and let $\cM$ consists of the class of $\sigdir$-directionally smooth distributions such that $\norm{\bz}_\infty \leq B$.  Let
    \begin{align}\label{eq:rhodefinition}
        \rho(\theta, \theta', \bz) = 2 \cdot \I\left[ \kbar_{\phibar}(\thetad, \bz) \neq \kbar_{\phibar}(\thetad', \bz) \right] + \max_{1 \leq k \leq K} \norm{\thetac^{(k)} - \thetac^{'(k)}}_1.
    \end{align}
    Then $\rho$ is a pseudo-metric satisfying the pseudo-isometry property with $\alpha = \frac{2A (B \vee 1)}{a \sigdir}$ and $\beta = 1$.  Furthermore, for all $\epsilon > 0$, $\bracknum\left( \Theta, \rho, \epsilon \right) \leq \left( \frac{9 A K^2 B D}{a \sigdir \epsilon} \right)^{K^2 (d + 1)}$.
\end{theorem}
We prove \Cref{thm:pwabracketing} in full detail in \Cref{app:pwabracketing}.  The proofs of both statements rely on the same key step, given in Lemma \ref{lem:kbarcontinuity}, which demonstrates that for fixed $\thetad^0$, even though the event $\I\left[ \kbar_{\phibar}(\thetad, \bz) \neq \kbar_{\phibar}(\thetad^0) \right]$ is not a continuous function of $\thetad$, its expectation is Lipschitz if $\bz$ is $\sigdir$-directionally smooth.  Thus, Lemma \ref{lem:kbarcontinuity} is a vast generalization of the motivating argument involving one-dimensional thresholds in \Cref{sec:genbrackets}.  This key lemma is proven by appealing to the anti-concentration of affine functions applied to directionally smooth random variables and is the only place that the analysis of $\ellbar$ is different from that of the function $\ell$ in \eqref{eq:piecewisecontinuous}.  We then use this result to both imply pseudo-isometry and to show that a cover of $\Theta$ with respect to $\ell_1$ gives rise to a generalized $\epsilon$-bracket with respect to $\cM$ and $\rho$.

Using \Cref{thm:pwabracketing}, we are able to prove a concrete regret bound for Algorithm \ref{alg:lazyftplexponential} on the class of piecewise continuous functions with affine boundaries:
\begin{corollary}\label{cor:pwaftpl}
    Suppose that $\ellbar$ is as in \eqref{eq:piecewisecontinuoustournaments} with $\phibar$ and $\Theta$ as in \Cref{thm:pwabracketing} with $B \geq 1$ and $\ellbar$ uniformly bounded in magnitude by $1$.  If we set $\eta = \BigThetaTil{ \left( T K^2 d D BA (a\sigdir)^{-1}  \right)^{2/3}}$ and $n = \sqrt{\eta}$, then Algorithm \ref{alg:lazyftplexponential} experiences $ \ee\left[ \reg_T \right] \leq \BigOhTil{\left( TA K^2 d B D (a\sigdir)^{-1} \right)^{2/3}  }$. 
    In particular, to achieve average regret $\epsilon$, it suffices to call $\ermoracle$ only $\BigOhTil{\frac{AK^2 dD B}{a \sigdir \epsilon^2}}$ times.
\end{corollary}
The proof of Corollary \ref{cor:pwaftpl} can be found in \Cref{app:pwaftplproof} and follows almost immediately from Theorems \ref{thm:lazyftplexponential} and \ref{thm:pwabracketing}.  The simplest example of a link function is simply to let $\psi(x) = x$ the identity, in which case we obtain a regret bound for piecewise continuous functions with affine boundaries.  


\subsection{Piecewise Continuous Prediction with Polynomial Boundaries}\label{sec:polies}

In order to broaden the scope of applications, we now consider more general boundaries between regions.  As mentioned above, the key to proving an analogue of \Cref{thm:pwabracketing} is the anti-concentration of affine functions applied to directionally smooth random variables.  While anti-concentration properties of more general functions remain an active area of research, sub-classes of polynomials, such as multi-linear functions of independent variables, are known to anti-concentrate in great generality \citep{mossel2010noise} and suffice to extend our results loss functions with these decision boundaries using our techniques, we instead focus on general polynomial boundaries and further restrict $\cM$:
\begin{definition}\label{def:polysmooth}
    For a polynomial $f: \rr^d \to \rr$ such that $f(x) = \sum_{\cI \subset [n]} \alpha_\cI x^{\cI}$, let $r = \deg(f) = \max \left\{ \abs{\cI} | \alpha_{\cI} \neq 0 \right\}$ denote the degree and let $\coeff_r(f) = \sqrt{\sum_{\abs{\cI} = r} \alpha_{\cI}^2}$ be the Euclidean norm of the vector of coefficients on the top-degree terms of the polynomial $f$.  We say that a distribution $\nu$ is $\sigpoly{r}$-polynomially smooth if for all $a \in \rr$, and all degree $f$ polynomials such that $\coeff_r(f) = 1$, it holds that $\pp_{x \sim \nu}\left( \abs{f(x) - a} \leq \epsilon \right) \leq \frac{\epsilon^{\frac 1r}}{\sigpoly{r}}$.
\end{definition}
Before proceeding, a few remarks are in order.  First, we note that directional smoothness is \emph{not} sufficient to ensure polynomial smoothness, as exhibited by \citet[Example 3]{glazer2022anti} and thus more constrained adversaries are indeed necessary to apply our methods.  Second, we obseve that Definition \ref{def:polysmooth} extends the notion of directional smoothness, with the latter corresponding to $\sigpoly{1}$-smoothness.  Finally, we observe that several common families of distributions are easily seen to be polynomially smooth with dimension-independent $\sigpoly{r}$, such as Gaussians and, more generally, product measures of log-concave marginals \citep[Corollary 4]{glazer2022anti}; we expand on this discussion in \Cref{app:polysmooth}.  Assuming an adversary is polynomially smooth, we prove an analogue of \Cref{thm:pwabracketing}, which then results in the following regret bound:
\begin{theorem}\label{thm:polyregret}
    Suppose $\cZ \subset \rr^d$ and $\Theta$ is a subset of Euclidean space with $\ell_1$ diameter bounded by $D$.  Let $\Thetad$ parameterize the set of tuples of $\binom{K}{2}$ degree $r$ polynomials $(f_{\bw_{kk'}})$ on $\rr^d$ such that $\coeff_r(f_{\bw_{kk'}}) = 1$ for all $k \in [K]$.  Suppose that $\ellbar$ is defined as in \eqref{eq:piecewisecontinuoustournaments} with $\phibar(\thetad, k, k', \bz) = f_{\bw_{kk'}}(\bz)$ and $\ellbar$ bounded in the unit interval. If $\cM$ is the class of $\sigpoly{r}$-polynomially smooth distributions such that $\norm{\bz}_\infty \leq B$ almost surely, then with the correct choices of $\eta, n$ given in \Cref{app:polies}, Algorithm \ref{alg:lazyftplexponential} experiences $\ee\left[ \reg_T \right] \leq \BigOhTil{\left(T K^2 r^2 d^r D B\sigpoly{r}^{-1} \right)^{\frac{4r - 2}{4r - 1}}}$.  Thus, the oracle complexity of achieving average regret at most $\epsilon$ is controlled by $\BigOhTil{\left( \frac{K^2 r^2 d^r D B}{\sigpoly{r}} \cdot \epsilon^{-2 r} \right)}$.
\end{theorem}
We prove Theorem \ref{thm:polyregret} in similarly to how we prove Theorem \ref{cor:pwaftpl}, i.e., we show an analogue of Theorem \ref{thm:pwabracketing} for polynomially smooth distributions to control the generalized bracketing numbers and pseudo-isometry constants with respect to $\rho$ from \eqref{eq:rhodefinition} before applying Theorem \ref{thm:lazyftplexponential}.  The full details are in \Cref{app:polies}.  We remark that the common thread between the proofs of \Cref{thm:polyregret} and \Cref{cor:pwaftpl} is that functions of random variables samples from distributions in $\cM$ are sufficiently anti-concentrated as to smooth the non-continuous parts of the loss functions.  Finally, note that we can replace $\ellbar$ with $\ell$ from \eqref{eq:piecewisecontinuous} with a similar margin assumption as in \Cref{app:margin}.

\newcommand{\dimx}{m}
\newcommand{\dimu}{d}
\newcommand{\bpsi}{\bm{\psi}}
\newcommand{\bzeta}{\bm{\zeta}}
\newcommand{\lv}{\ell^{v}}
\newcommand{\bxi}{\bm{\xi}}
\newcommand{\boldeta}{\bm{\eta}}

\section{Smoothed Multi-Step Planning}\label{sec:planning}
In previous sections, we were interested in online prediction; here we focus on the related problem of multi-step decision making. Specifically, we study the setting of multi-step planning, where the learner plays a sequence of dynamical inputs (in control parlance, an \emph{open-loop plan}) to minimize a cumulative control loss over a finite planning horizon. We focus on ``hybrid dynamics'' \citep{borrelli2003constrained,henzinger1998hybrid}, where each state space is partitioned into regions (called modes) within which the dynamics are Lipschitz.  We consider the case of affine decision boundaries between modes here and defer discussion of polynomial boundaries to \Cref{app:planning}.  We remark that this problem is challenging due to the introduction of possible discontinuities across modes, again limiting the applicability of previous techniques. This class is rich enough to model piecewise-affine dynamics frequently encountered in robotic-planning \citep{hogan2016feedback,anitescu1997formulating,aydinoglu2021stabilization}; in the appendix, we generalize further to  polynomial decision boundaries \citep{posa2015stability}. See also the related work in \Cref{app:relatedwork}.

Formally, we fix a planning horizon $H \in \bbN$ and consider a family of dynamical systems with states $\bx_h \in \cX \subset \R^{\dimx}$ and inputs $\bu_h \in \cU \subset \R^{\dimu}$. Our decision variables are \emph{plans} $\theta = \bbaru_{1:H} \in \cK \subset \cU^{\times H}$ and our context are tuples $z_t = (\bx_{t,1},\boldeta_{t,1:H},\bxi_{t,1:H},g_{t;1:H,1:K},\lv_t,\bW_{t,1:H})$ consisting of an initial state $\bx_{t,1} \in \cX$, noises $\boldeta_{t,h} \in \cX$  and $\bxi_{t,h} \in \cU$, continuous functions $g_{t,h,k}$ defining the dynamics for mode-$k$ at step $h$, time-dependent continuous losses $\lv_t$, and matrices $\bW_{t,h} \in \R^{K(\dimx + \dimu+1)}$ determining the boundaries between modes, where $\bW_{t,h}$ has rows $\bw_{t,h,k} \in \cS^{\dimx + \dimu}$.  We use $\bv \in \cV = \cX \times \cU$ to denote concatenations of state and input.  We suppose piecewise-continuous dynamics, where
\begin{align}
    \bx_{t,h+1}(\theta ) &= g_{t,h,k_{t,h}(\bv_{t,h}(\theta))}(\bv_{t,h}(\theta)) + \boldeta_{t,h}, \quad \text{and } \label{eq:dynamics} \\
    \bu_{t,h}(\theta) &= \bbaru_{t,h} + \bxi_{t,h},  \quad \bv_{t,h}(\theta) = (\bx_{t,h}(\theta), \quad \bu_{t,h}(\theta)), \\
k_{t,h}(\bv) &= \argmax_{k \in [K]} \phi_{t,h}(k, \bv),  \quad \text{and}\quad \phi_{t,h}(k, \bv) = \inprod{\bw_{t,h,k}}{(\bv,1)}. 
\end{align}

In words, for each time $t$, there are length $H$ trajectories that evolve according to piecewise continuous dynamics, where each piece (mode) is determined by affine functions of both the previous state and an input.  We aim to minimize regret with the loss $\ell(\theta,z_t) :=  \lv_t(\bv_{t,1:H}(\theta))$, where $\lv_t:\cV^H \to \R$ are $1$- Lipschitz functions of both the state and input.  We assume that, \emph{for fixed mode sequences $k_{1:h} \in [K]^h$}, the $h \in [H]$-fold compositions of the Lipschitz dynamic maps $g_{t,h,k_h}\circ g_{t,h-1,k_{h-1}} \circ \cdots \circ g_{t,1,k_1}$ are $L$-Lipschitz as functions of $\theta \in \cK$ in an $\ell_1 \to \ell_1$ sense (see the appendix for a precise statement). Though $L$ may be exponential in $H$ in the worst-case, common stability conditions ensure that $L$ is more reasonably bounded; for further elaboration, see Remark~\ref{rem:L_scaling}. Finally, in order to incorporate smoothness, let $\filt_{t}$ denote the filtration generated by $(z_{1:t-1},\lv_{t},g_{t,1:H,1:K},\bW_{t,1:H})$, and for $h \geq 0$ let $\filt_{t,h}$ denote the filtration generated by $\cF_{t}$ and $\bxi_{t,1:h},\boldeta_{t,1:h},\bx_{t,1}$; we suppose that the tuple $(\bxi_{t,h}, \boldeta_{t,h})$ of dynamics and input noise, conditioned on $\filt_{t,h}$, is $\sigdir$-directionally smooth.  

While the restriction to open-loop plans may seem limiting, we note that the flexibility in our definition of the $g_{t,h,k}$ allows us to incorporate a wide variety of state-dependent policies with minimal modification.  For example, our framework includes the popular setting of linear controls, where the learner plays an affine function mapping the state to an input; by letting $g_{t,h,k}$ be multilinear in the input matrix and the state and letting the loss be quadratic, both of which remain Lipschitz due to our boundedness assumptions, we naturally recover a piecewise generalization of the well-known Linear Quadratic Regulator (LQR).  Our main result is the following.
\begin{theorem}\label{thm:planning} 
    Suppose that we are in the situation described by \eqref{eq:dynamics}, with $(\boldeta_{t,h},\bxi_{t,h}) | \filt_{t,h-1}$ $\sigdir$-directionally smooth, $\sup_{\bv \in \cV} \norm{\bv}_1 \leq D$, the $\ell_t^v$ are Lipschitz and bounded, and the $g_{t,h,k}$ satisfying technical continuity assumptions found in \Cref{thm:planningformal}.  If there is some margin parameter $\gamma > 0$ such that for all $t \in [T]$ and $h \in [H]$ it holds that $\min_{k\neq k' \in [K]} \norm{\bw_{t,h,k} - \bw_{t,h,k'}} \geq \gamma$ and the planner plays $\bbaru_{t,h}$ according to Algorithm \ref{alg:lazyftplexponential}, then the oracle complexity of achieving average regret $\epsilon$ is at most $\BigOhTil{(d H^5 K^4 (D L/ (\gamma \sigdir))^2)^{\frac 13} \epsilon^{-2}}$.
\end{theorem}
The proof, elaboration of assumptions, and the extension to polynomial decision boundaries are given in \Cref{app:planning}. The proof follows the template of the previous section; to handle the multi-step setup, we argue that smooth dynamical noise suffices to ensure that, when $\theta,\theta' \in \cK$ are sufficiently close, smoothness ensures that the sequence of modes $k_{t,h}(\bv_{t,h}(\theta)),k_{t,h}(\bv_{t,h}(\theta))$ coincide for all $h \in [H]$ with high probability; this requires a telescoping argument similar in spirit to the performance-difference lemma in reinforcement learning \citep{kakade2003sample}.



\section*{Acknowledgments}
AB acknowledges support from the National Science Foundation Graduate Research Fellowship under Grant No. 1122374. We also acknowledge support from ONR under grant N00014-20-1-2336, DOE under grant DE-SC0022199, and NSF through award DMS-2031883.  MS acknowledges support from Amazon.com Services LLC grant; PO 2D-06310236. We also acknowledge Russ Tedrake, Terry H.J. Suh, and Tao Pang for their helpful comments.

\bibliographystyle{plainnat}
\bibliography{refs}
\appendix


\section{Related Work}\label{app:relatedwork}
In this section, we continue our discussion of relevant related work from the introduction.

\paragraph{Smoothed Online Learning}

Smoothed online learning was originally proposed in \citet{rakhlin2011online}, with more recent work including \citet{haghtalab2020smoothed,haghtalab2022oracle,haghtalab2022smoothed,block2022smoothed}.  In particular, \citet{haghtalab2022smoothed} characterized the statistical rates for smoothed online classification and \citet{block2022smoothed} did the same for the more general setting of real-valued functions.  The first analysis of  oracle-efficient algorithms for smoothed online learning was conducted in \citet{haghtalab2022oracle,block2022smoothed}, with both works providing both proper and improper algorithms.  Both works also provided lower bounds, showing the exponential gap in dependence on the smoothing paramter $\sigma$ of the regret incurred by inefficient and oracle-efficient algorithms.

In order to address the exponentially worse regret guarantees for oracle-efficient smoothed online learning, \citet{block2022efficient} examined a special case where the loss function is a linear threshold function, parameterized by elements in $\Theta$.  In the even more restricted, realizable setting, where there exists some $\theta \in \Theta$ achieving zero cumulative loss, that work was able to recover regret logarithmic in $T / \sigma$.  Unfortunately, the noiseless assumption is unrealistic and the resulting algorithm is not robust to its removal; thus, \citet{block2022efficient} proposed a new notion, directional smoothness, that relaxed the smoothness assumption to one more specifically suited to linear structure.  Building on this work, \citet{block2023smoothed} demonstrated that oracle-efficient smoothed online learning was possible in the challenging Piecewise Affine (PWA) setting, with regret depending only polynomially on all relevant problem parameters, albeit with a somewhat impractical algorithm and a significantly worse dependence on the horizon in the regret.  We note, however, that the results of \citet{block2023smoothed} are not comparable with our results because our algorithm requires a stronger notion of ERM oracle than that of \citet{block2023smoothed}, a point on which we elaborate below.

\paragraph{Follow the Perturbed Leader and Oracle-Efficient Online Learning}  Follow the Perturbed Leader (FTPL) was first proposed and analyzed in \citet{kalai2005efficient} for the setting of linear losses.  In that work, the authors introduced the Be-the-Leader lemma, decomposing regret into a perturbation term and a stability term, which remains the most popular way to prove regret bounds for such algorithms.  Since then, the algorithmic framework has seen much popularity, with applications to multi-armed bandits \citep{abernethy2015fighting}, Reinforcement Learning \citep{daifollow}, and online structured learning \citep{Cohen2015following}, among others.  Of greater relevance to this paper, are the works \citet{agarwal2019learning,suggala2020online}, which demonstrate that in the adversarial online learning setting, if the loss functions are Lipschitz, then FTPL with an exponential perturbation can attain optimal regret.  In our \Cref{thm:lazyftplexponential}, we extend the approach of these two works beyond the Lipschitz case, using our new notion of complexity.  Due to the memoryless property of the exponential distribution, it is one of the most popular perturbations used for analysis of FTPL instantiations \citep{kalai2005efficient,suggala2020online,agarwal2019learning}, although other distributions have been studied with different techniques \citep{abernethy2014online,abernethy2015fighting,li2017beyond,block2022smoothed,haghtalab2022oracle}.  Our analysis heavily relies on this memoryless property and thus we restrict our focus to this instantiation, leaving as an interesting question for future work whether similar results can hold with more general perturbation distributions.

Lower bounds for oracle-efficient online learning have proven substantially more difficult than upper bounds.  In \citet{hazan2016computational}, the authors demonstrated an exponential gap in the statistical and computational complexities of achieving average regret at most $\epsilon$; similarly, the lower bounds of \citet{block2022smoothed,haghtalab2022oracle} are based on reductions to this result.  On the other hand, this lower bound appears somewhat brittle, as it applies only to proper learning with a somewhat restricted notion of ERM oracle.  While the oracle used in \citet{block2023smoothed} fits into this model, the oracle we assume, as well as that used in \citet{hazan2016computational,suggala2020online}, does not.  For more discussion on this point, see \citet{hazan2016computational}.

\paragraph{Prediction and Planning in Piecewise Affine Systems} Our examples are motivated in part by the planning and prediction in piecewise affine systems, and more generally, systems with polynomial boundaries between Lipschitz regions. Piecewise affine dynamics are popular in the constrained MPC and hybrid systems literature, \citep{henzinger1998hybrid,borrelli2003constrained}, due in part to their ability to model contact dynamics in robotic systems \citep{marcucci2019mixed,anitescu1997formulating,suh2022bundled}; polynomial boundaries are studied in \citep{posa2015stability}. \cite{suh2022bundled,suh2022bundled} have studied the advantages of randomized noise injection for trajectory planning through systems with discontinuities of these forms, demonstrating numerous advantages. For typical noise distributions (e.g. Gaussian), these randomized noise injections introduce the same sorts of smoothness properties leveraged in the present work. 

\paragraph{Statistical and Online Learning for Control and Dynamical Prediction.}
Building on the decades-old literature for system-identification \citep{ljung1999system,ljung1992asymptotic,deistler1995consistency}, recent work has provided finite-sample statistical guarantees for parameter recovery in linear dynamical systems for various regimes of interest \citep{simchowitz2018learning,dean2017sample,oymak2019non,tsiamis2019finite,tsiamis2022statistical}. Further research has studied smooth nonlinear dynamics \citep{mania2020active,sattar2020non,foster2020learning}, and settings where only the observation model is nonlinear \citep{dean2021certainty,mhammedi2020learning}.  Relevant to this work, \cite{sattar2021identification} study \emph{Markov jump systems}, where the system dynamics alternate between one of a finite number of linear systems (``modes''), and switches between modes are governed by a (discrete) Markov chain. In constrast, the dynamics with piecewise affine boundaries studied in this work have modes which depend on state.

In addition to the recent advances in finite-sample system identification, a vast body of work has studied linear control tasks from the perspective of regret \citep{abbasi2011regret,agarwal2019online,simchowitz2020improper,simchowitz2020naive,cohen2018online}. \cite{kakade2020information} studies nonlinear online control of \emph{fixed } nonlinear systems with a certain linear-parametric structure; similarly to the present work, although in less generality, a crucial step is their use of Gaussian smoothing to guarantee low regret. In contrast, \Cref{sec:planning} allows for time-varying dynamics and does not rely on recovery of a low-dimensional parameter.

Our guarantees for prediction are similar in spirit to online prediction for linear control settings attained \citep{hazan2017learning,tsiamis2020sample}; though, of course, they pertain to a far broader class of dynamical systems.

\paragraph{Bracketing Entropy} The notion of bracketing number originally dates back to \citet{blum1955convergence,dehardt1971generalizations} and was used to prove uniform laws of large numbers.  They were then used by \citet{dudley1978central} to prove uniform central limit theorems.  There exist many bounds on bracketing numbers for concrete function classes of interest, with the most notable likely being Besov and Sobolev classes \citep{nickl2007bracketing}.  Our Definition \ref{def:genbrackets} generalizes this notion both by changing the absolute value to a general pseudo-metric and, more importantly, forcing the expectation to be uniform over a family of measures.

\paragraph{PAC Learning Piecewise-Lipschitz Functions and the ``Dispersion'' Condition} We stress that we consider \emph{online learning} of piecewise-Lipschitz functions. In the \emph{PAC learning} setting, it is possible to derive on-distribution generalization bounds by bounding the pseudo-dimension of piecewise-affine functions \citep{balcan2021much}. Moreover, our class of functions do not satisfy the ``dispersion condition'' \citep{balcan2018dispersion}, which instead would render the class online-learnable without smoothing; further still, it remains an open problem as to how to develop \emph{oracle-efficient} online learning algorithms based on dispersion.  
\section{Proof of Proposition \ref{prop:singlestepbracketing}} \label{app:singlestepbracketing}

In this section, we prove Proposition \ref{prop:singlestepbracketing} by appealing to Freedman's inequality.  We recall:
\begin{lemma}[Freedman's Inequality, \citet{agarwal2014taming}]\label{lem:freedman}
    Let $Z_t$ for $1 \leq t \leq T$ be a real-valued martingale difference sequence such that, conditional on $Z_{1:t-1}$, almost surely $\abs{Z_t} \leq R$.  Then for any $0 < \eta < \frac 1R$, with probability at least $1 - \delta$, it holds that
    \begin{align}
        \sum_{t = 1}^T Z_t \leq \eta \cdot \sum_{t = 1}^T \ee_{t-1}\left[Z_t^2\right] + \frac{R \log\left( \frac 1\delta \right)}{\eta}.
    \end{align}
\end{lemma}
We are now ready to prove the key result:
\begin{proof}[Proof of Proposition \ref{prop:singlestepbracketing}]\label{app:keybound}
    Let
    \begin{align}
       Z_i(\theta, \theta')  = \rho(\theta, \theta', x_i) - \ee_i\left[ \rho(\theta, \theta', x_i) \right],
    \end{align}
    where we use the convenient shorthand $\ee_i\left[ \cdot \right] = \ee\left[ \cdot | \filt_{i-1}\right]$, with $\filt_i$ as in the statement of the proposition.  We begin by observing that by assumption, $\abs{Z_i} \leq 2 D$ and further, that
    \begin{align}
        \ee_i\left[ Z_i^2 \right] &\leq \ee_i\left[ \rho(\theta, \theta', x_i)^2 \right] \leq D \cdot \ee_i\left[ \rho(\theta, \theta', x_i) \right].
    \end{align} 
    Applying Lemma \ref{lem:freedman}, we see that for any fixed $0 < \eta < \frac 1R$ and $\theta, \theta' \in \Theta$, with probability at least $1 - \frac{\delta}{2}$ it holds that
    \begin{align}
        \sum_{i  =1}^n \rho(\theta, \theta', x_i) &\leq (1 + \eta D) \sum_{i = 1}^n \ee_i\left[ \rho(\theta, \theta', x_i) \right] + \frac{2 D \log\left( \frac 2\delta \right)}{\eta} \\
        &\leq (1 + \eta D) n \cdot \sup_{\nu \in \cM} \ee_\nu\left[ \rho(\theta, \theta', x) \right] + \frac{2D \log\left( \frac{2}{\delta} \right)}{\eta}.
    \end{align}
    Let $\cN = \left\{ (\theta_j, \cB_j) \right\}$ denote a minimal generalized $\epsilon$-bracket.  By a union bound and setting $\eta = \frac 1{3D}$, we see that with probability at least $1 - \frac{\delta}{2}$, it holds for all $\theta_j, \theta_k \in \cN$ that
    \begin{align}
        \sum_{i = 1}^n \rho(\theta_j, \theta_k, x_i) \leq 4 n \cdot \sup_{\nu \in \cM} \ee_\nu\left[ \rho(\theta_j, \theta_k, x) \right] + 2 D^2 \cdot \log\left( \frac{2\abs{\cN}^2}{\delta} \right).
    \end{align}
    Similarly, we define
    \begin{align}
        \Ztil_i^j = \sup_{\theta \in \cB_j} \rho(\theta, \theta_j, x_i) - \ee_i\left[ \sup_{\theta \in \cB_j} \rho(\theta, \theta_j, x_i)  \right]
    \end{align}
    and note that by the definition of the generalized bracket,
    \begin{align}
        \abs{\Ztil_i^j} \leq 2D && \ee_i\left[ (\Ztil_i^j)^2 \right] \leq D \epsilon.
    \end{align}
    Thus, again applying Lemma \ref{lem:freedman} and a union bound, it holds that with probability at least $1 - \frac{\delta}{2}$, for all $\theta_j \in \cN$,
    \begin{align}
        \sup_{\theta \in \cB_j} \sum_{i = 1}^n \rho(\theta, \theta_j, x_i) \leq 4 n \epsilon + 2 D^2 \log\left( \frac{2 \abs{\cN}^2}{\delta} \right).
    \end{align}
    By the triangle inequality and a union bound, we then have that with probability at least $1 - \delta$, if for all $\theta, \theta' \in \Theta$, we let $\theta_j$ be the projection of $\theta$ to $\cN$ and $\theta_k$ the projection of $\theta'$ to $\cN$, it holds that 
    \begin{align}
        \sum_{i = 1}^n \rho(\theta, \theta', x_i) &\leq \inf_{\theta_j \in \cN} \left\{ \sum_{i = 1}^n \rho(\theta, \theta_j, x_i) \right\} + \sum_{i = 1}^n \rho(\theta_j, \theta_k, x_i) + \inf_{\theta_k \in \cN} \left\{ \sum_{i = 1}^n \rho(\theta_k, \theta', x_i) \right\} \\
        &\leq 4n \sup_{\nu \in \cM} \ee_\nu\left[ \rho(\theta_j, \theta_k, x) \right] + 4n \epsilon + 6 D^2 \log\left( \frac{2 \abs{\cN}^2}{\delta} \right) \\
        &\leq 4 n \sup_{\nu \in \cM} \ee_\nu\left[ \rho(\theta, \theta', x) \right]  + \left( 2 n \sup_{\nu \in \cM} \ee_\nu\left[ \rho(\theta, \theta_j, x) \right] + 2 n \sup_{\nu \in \cM} \ee_\nu\left[ \rho(\theta_k, \theta', x) \right]   \right) \\
        &+  4n \epsilon + 6 D^2 \log\left( \frac{2 \abs{\cN}^2}{\delta} \right) \\
        &\leq 4n \sup_{\nu \in \cM}\ee_\nu\left[ \rho(\theta, \theta', x) \right]  +8n \epsilon + 6 D^2 \log\left( \frac{2 \abs{\cN}^2}{\delta} \right).
    \end{align}
    The result follows.
\end{proof}
\section{Proof of Proposition \ref{prop:lazyftpl}}\label{app:lazyftplmaster}

In this section, we prove a more general version of Proposition \ref{prop:lazyftpl} by combining the classic Be-the-Leader Lemma from \citet{kalai2005efficient} with our Proposition \ref{prop:singlestepbracketing}.  We begin by stating and proving a lazy version of the Be-the-Leader lemma.  We follow the proof of \citet[Lemma 31]{block2022smoothed}:
\begin{lemma}[Be-The-Leader Lemma]\label{lem:btl}
    Let $n \in \bbN$ and suppose for each $1 \leq \tau \leq T / n$, the learner chooses some approximate minimizer of the perturbed cumulative loss.  More precisely, for some real-valued function $\gamma$ on the stochastic process, the learner chooses some $\thetatil_\tau$ satisfying
    \begin{align}\label{eq:ftpl}
        L_{(\tau - 1) \cdot n}(\thetatil_\tau) + \omega_{(\tau - 1) \cdot n + 1 }(\thetatil_\tau) \leq \gamma(\omega_{(\tau - 1) \cdot n + 1} ) + \inf_{\theta \in \Theta} L_{(\tau - 1) \cdot n}(\theta) + \omega_{(\tau - 1) \cdot n + 1 }(\theta)
    \end{align}
    and plays $\thetatil_\tau$ for all $(\tau - 1) \cdot n + 1 \leq t \leq \tau  \cdot n$.  Suppose that the $\omega_t$ are independent across $t$ and identically distributed random processes on $\Theta$ satisfying $\ee\left[ \sup_{\theta \in \Theta} \omega_t(\theta) \right] \geq 0$.  Then, the learner experiences the following regret:
    \begin{align}
        \ee\left[ \reg_T \right] \leq  \frac{T}{n} \cdot \ee\left[ \gamma(\omega_1) \right] +  \ee\left[ \sup_{\theta} \omega_1(\theta) \right] + \sum_{\tau = 1}^{T / n} \ee\left[ \sum_{t = (\tau - 1) \cdot n + 1}^{\tau \cdot n} \ell(\thetatil_\tau, z_t) - \ell(\thetatil_{\tau+1}, z_t) \right].
    \end{align}
\end{lemma}
\begin{proof}
    We apply the proof of \citet[Lemma 31]{block2022smoothed} to the cumulative loss over $n$ steps.  Thus, for each $1 \leq \tau \leq T' = \floor{T / n}$, recall that
    \begin{align}
        \elltil_\tau(\theta) = \sum_{t = (\tau - 1) \cdot n + 1}^{\tau \cdot n} \ell(\theta, z_t).
    \end{align}
    We will apply induction on $T'$ to the inequality
    \begin{align}\label{eq:inductionbtl}
        \ee\left[ \sum_{\tau = 1}^{T'} \elltil_\tau(\theta_{\tau + 1}) \right] \leq \ee\left[ \sum_{\tau = 1}^{T'} \ell_\tau(\theta_{T' + 1}) + \omega_{T'\cdot n + 1}(\theta_{T' + 1}) \right] + \ee\left[ \sup_{\theta \in \Theta} \omega_{1}(\theta) \right] + \frac{T'}{n} \cdot \ee\left[ \gamma(\omega_1) \right].
    \end{align}
    For the base case of $T' = 0$ the statement is trivial.  Suppose that for some fixed $T' -1$ that \eqref{eq:inductionbtl} holds.  Then we see by construction that
    \begin{align}
        \ee\left[ \sum_{\tau = 1}^{T' - 1} \elltil_\tau(\theta_{T'}) + \omega_{T'\cdot n + 1}(\theta_{T'}) \right] &\leq \ee\left[ \inf_{\theta \in \Theta} \sum_{\tau = 1}^{T' - 1} \elltil_\tau(\theta) + \omega_{T' \cdot n + 1}(\theta) \right] + \ee\left[ \gamma(\omega_{T'\cdot n + 1}) \right]\\
        &\leq \ee\left[ \sum_{\tau = 1}^{T' - 1} \elltil\left( \theta_{T' + 1} \right) + \omega_{(T' + 1) \cdot n + 1}\left( \theta_{T' + 1} \right) \right] + \ee\left[ \gamma(\omega_{(T' + 1) \cdot n + 1}) \right],
    \end{align}
    where the second inequality follows from the fact that the $\omega_t$ are independent and identically distributed as well as the construction of $\theta_{T' + 1}$.  Combining the induction hypothesis \eqref{eq:inductionbtl} with the above inequality tells us that
    \begin{align}
        \ee\left[ \sum_{\tau = 1}^{T' - 1} \elltil_\tau(\theta_{\tau + 1}) \right] \leq \ee\left[ \sum_{\tau = 1}^{T' - 1} \elltil_\tau\left( \theta_{T' + 1} \right) + \omega_{(T' + 1) \cdot n + 1}\left( \theta_{T' + 1} \right) \right] + \ee\left[ \sup_{\theta \in \Theta} \omega_1(\theta) \right] + T' \cdot \ee\left[ \gamma(\omega_1) \right].
    \end{align}
    Adding $\ee\left[ \elltil_{T'}(\theta_{T' + 1}) \right]$ to both sides finishes the induction proof.

    To continue, we compute:
    \begin{align}
        \ee\left[ \sum_{\tau = 1}^{T'} \elltil_\tau(\theta_{T' + 1}) + \omega_{T'  + 1}(\theta_{T' + 1}) \right] &\leq \ee\left[ \inf_{\theta \in \Theta} \sum_{\tau = 1}^{T'} \elltil_\tau\left( \theta \right) + \sup_{\theta' \in \Theta} \omega_1(\theta) \right] +\frac{T'}{n} \cdot \ee\left[ \gamma(\omega_1) \right]\\
        &\leq \ee\left[ \inf_{\theta \in \Theta} \sum_{t = 1}^T \ell(\theta, z_t) \right] + \ee\left[ \sup_{\theta \in \Theta} \omega_1(\theta) \right]\frac{T'}{n} \cdot \ee\left[ \gamma(\omega_1) \right],
    \end{align}
    where we used the construction of $\theta_{T' + 1}$ for the first inequality and the definition of $\elltil_\tau$ for the second inequality.  To conclude, we apply \eqref{eq:inductionbtl} and observe:
    \begin{align}
        \ee\left[ \reg_T \right] &= \ee\left[ \sum_{t = 1}^T \ell(\theta_{\tau(t)}) - \inf_{\theta \in \Theta} \sum_{t = 1}^T \ell(\theta, z_t) \right] \\
        &\leq \ee\left[ \sum_{\tau = 1}^{T'} \elltil_\tau(\theta_\tau) \right] - \ee\left[ \sum_{\tau = 1}^{T'} \elltil_\tau(\theta_{\tau + 1}) \right] + 2 \ee\left[ \sup_{\theta \in \Theta} \omega_1(\theta) \right] + \frac{T'}{n} \cdot \ee\left[ \gamma(\omega_1) \right],
    \end{align}
    where we denoted by $\tau(t) = \floor{t / n}$.  The result follows from the construction of $\elltil_\tau$.
\end{proof}
We are now ready to prove the main result of the section.
\begin{proposition}\label{prop:lazyftplgeneral}
    Suppose that we are in the constrained online learning setting, where the adversary is constrained to sample $z_t$ from some distribution in the class $\cM$.  Suppose further that there is a pseudo-metric $\rho$ on $\Theta$ parameterized by $\cZ$ satisfying the psuedo-isometry property of Definition \ref{def:pseudoisometry} such that for all $\theta, \theta' \in \Theta$ it holds that $\sup_{\nu \in \cM} \ee_\nu\left[ \ell(\theta, z) - \ell(\theta',z) \right] \leq \sup_{\nu \in \cM} \ee_\nu\left[ \rho(\theta, \theta', z) \right]$.  If the learner plays $\thetatil_\tau$ as in Lemma \ref{lem:btl} and $\sup_{\theta, \theta' \in \Theta} \rho(\theta, \theta', z) \leq D$, then for any $\epsilon > 0$, the expected regret is bounded as:
    \begin{align}
        \ee\left[ \reg_T \right] &\leq \frac{T}{n} \cdot \ee\left[ \gamma(\omega_1) \right] + \ee\left[ \sup_{\theta \in \Theta} \omega_1(\theta) \right]  +  8 \epsilon T + 1 \\
        &+\frac{6 TD^2}{n} \cdot \log\left( T \cdot \bracknum\left( \Theta, \rho, \epsilon \right) \right) + 4n \alpha \cdot \sum_{\tau = 1}^{T/n} \ee\left[ \norm{\thetatil_\tau - \thetatil_{\tau + 1}}^\beta \right] .
    \end{align}
\end{proposition}
\begin{proof}
    Applying Lemma \ref{lem:btl}, we have
    \begin{align}
        \ee\left[ \reg_T \right] \leq  \frac{T}{n} \cdot \ee\left[ \gamma(\omega_1) \right] +  \ee\left[ \sup_{\theta} \omega_1(\theta) \right] + \sum_{\tau = 1}^{T / n} \ee\left[ \sum_{t = (\tau - 1) \cdot n + 1}^{\tau \cdot n} \ell(\thetatil_\tau, z_t) - \ell(\thetatil_{\tau+1}, z_t) \right].
    \end{align}
    We thus only need to bound the final sum above.  By the fact that the loss function is Lipschitz with respect to the pseudo-metric, we have:
    \begin{align}
        \sum_{\tau = 1}^{T / n} \ee\left[ \sum_{t = (\tau - 1) \cdot n + 1}^{\tau \cdot n} \ell(\thetatil_\tau, z_t) - \ell(\thetatil_{\tau+1}, z_t) \right] &\leq \sum_{\tau = 1}^{T / n} \ee\left[ \sum_{t = (\tau - 1) \cdot n + 1}^{\tau \cdot n} \rho(\thetatil_\tau, \thetatil_{\tau+1}, z_t)\right].
    \end{align}
    Thus for any fixed $\tau$ and for all $\epsilon, \delta > 0$,
    \begin{align}
        \ee&\left[ \sum_{t = (\tau - 1) \cdot n + 1}^{\tau \cdot n} \rho(\thetatil_\tau, \thetatil_{\tau+1}, z_t)\right] \\
        &\leq \ee\left[4n \sup_{\nu \in \cM}\ee_\nu\left[ \rho(\thetatil_\tau, \thetatil_{\tau+1}, z)  \right] + 8 \epsilon n + 6 D^2 \log\left( \frac{2 \bracknum\left( \Theta, \rho, \epsilon \right)}{\delta} \right) + \delta n\right] \\
        &\leq 4 n \alpha \cdot \ee\left[\norm{\thetatil_\tau - \thetatil_{\tau+1}}^\beta  \right] + (8 \epsilon + \delta) n + 6 D^2 \cdot \log\left( \frac{2 \bracknum(\Theta, \rho, \epsilon)}{\delta} \right),
    \end{align}
    where the first inequality follows from Proposition \ref{prop:singlestepbracketing} and the second follows from assuming that $\rho$ satisfies the conditions of Definition \ref{def:pseudoisometry}.  Summing over $\tau$ and setting $\delta = T^{-1}$ concludes the proof.
\end{proof}
Finally, we observe that Proposition \ref{prop:lazyftpl} is a special case of the preceding analysis:
\begin{proof}[Proof of Proposition \ref{prop:lazyftpl}]
    The result follows immediately by taking $\gamma = 0$ uniformly in Proposition \ref{prop:lazyftplgeneral}.
\end{proof}

\section{Proof of Corollary \ref{cor:gaussianperturbation}}\label{app:gaussian}

In this section, we prove \Cref{cor:gaussianperturbation} by first demonstrating that generalized brackets in this setting can simply be taken to be classical brackets and then by applying a stability bound from \citet{block2022smoothed}.  In order to respect notational convention, we will replace $\Theta$ with a function class $\cF$ and consider functions $f \in \cF$ instead of parameters $\theta \in \Theta$.  We will let
\begin{align}
    \rho(f, f', z) = L \cdot \abs{f(x) - f'(x)}, && z = (x,y) \in \cZ,
\end{align}
and show that with this $\rho$, classical brackets become generalized brackets after rescaling:
\begin{lemma}\label{lem:classicalsmoothed}
    Let $\cM$ denote the class of distributions that are $\sigma$-smooth with respect to some distribution $\mu$ on $\cX$.  If $\rho$ is as above, then
    \begin{align}
        \bracknum(\cF, \rho, \epsilon) \leq \cN_{[]}\left( \cF, \mu, \sigma \epsilon / L\right).
    \end{align}
    Moreover, $\rho$ satisfies the pseudo-isometry for $\alpha = L \cdot \sigma^{-1}$ and $\beta = 1$ for the norm $L^1(\mu)$.
\end{lemma}
\begin{proof}
    Let $\cN = \left\{ \cB_i \right\}$ denote an $\epsilontil$-bracket, in the classical sense, of $\cF$ with respect to $\mu$, where $\epsilontil = \frac{\sigma \epsilon}{L}$, and let $f_i$ denote an arbitrary member of $\cB_i$.  Then we see for all $\nu \in \cM$,
    \begin{align}
        \ee_\nu\left[ \sup_{f \in \cB_i} \rho(f, f_i, z) \right] &= \ee_\nu\left[ \sup_{f \in \cB_i} \abs{f(x) - f_i(x)} \right] \\
        &= \ee_\mu\left[\frac{d\nu}{d\mu} \sup_{f \in \cB_i }\abs{f(x) - f_i(x)} \right] \\
        &\leq \frac 1\sigma \cdot \ee_\mu\left[ \sup_{f \in \cB_i} \abs{f(x) - f_i(x)} \right] \\
        &\leq \frac{\epsilon}{\sigma},
    \end{align}
    by definition of a classical bracket.  We conclude the proof of the first statement by observing that, again by definition, the $\cB_i$ cover $\cF$.

    The second statement is trivial by definition of smooth distributions.
\end{proof}
We now recall a stability result:
\begin{lemma}[Lemma 34 from \citet{block2022smoothed}]\label{lem:blocketal}
    Suppose that we are in the setting of Corollary \ref{cor:gaussianperturbation} and let $\muhat$ denote the empirical measure on the sampled $x_i$.  If the function $\ell$ is $L$-Lipschitz in the first argument and
    \begin{align}\label{eq:normcomparison}
        \sup_{f, f' \in \cF} \abs{\norm{f - f'}_{L^2(\mu)}^2 - \norm{f - f'}_{L^2(\muhat)}^2} \leq \Delta,
    \end{align}
    then for any fixed $y$,
    \begin{align}
        \ee\left[ \norm{\ell(f_t(\cdot), y) - \ell(f_{t+1}(\cdot), y)}_{L^1(p_t)} \right] \leq \frac{30 L^3 \log(\eta)}{\sigma \eta} \cdot \ee\left[ 1 + \sup_{f \in \cF} \omega(f) \right] + \frac{2 L \Delta}{\sigma}.
    \end{align}
\end{lemma}
The assumption in Lemma \ref{lem:blocketal} that the empirical and population norms are close to each other is a standard consequence of classical learning theory.  We are thus ready to provide the main proof:
\begin{proof}[Proof of Corollary \ref{cor:gaussianperturbation}]
    By Proposition \ref{prop:lazyftpl}, it holds that Algorithm \ref{alg:lazyftpl} experiences
    \begin{align}
        \ee\left[ \reg_T \right] \leq \BigOh{\ee\left[ \sup_{f \in \cF} \omega(f) \right] + \epsilon T + \frac{T}{n} \cdot \log\left( T \cdot \bracknum(\cF, \rho, \epsilon) \right) + 2 T \alpha \max_{\tau \leq T/n} \ee\left[ \norm{f_\tau - f_{\tau + 1}}^\beta \right]}.
    \end{align}
    By the results of Lemma \ref{lem:classicalsmoothed}, we may take $\alpha = \frac{1}{\sigma}$ and $\beta = 1$ above to recover
    \begin{align}
        \ee\left[ \reg_T \right] \leq \BigOh{\ee\left[ \sup_{f \in \cF} \omega(f) \right] + \epsilon T + \frac{T}{n} \cdot \log\left( T \cdot \cN_{[]}\left( \cF, \mu, \frac{\sigma \epsilon}{L} \right) \right) + \frac{T}{\sigma} \cdot \max_{\tau} \ee\left[ \norm{f_\tau - f_{\tau + 1}} \right]}.
    \end{align}
    Observe now that if $\ell$ is $L$-Lipschitz, then $\elltil$, the cumulative loss over an epoch of length $n$, is $Ln$-Lipschitz by the triangle inequality.  Thus, we see that
    \begin{align}
        \max_{\tau} \ee\left[ \norm{f_\tau - f_{\tau + 1}} \right] \leq \BigOh{\frac{L^3 n^3 \log(\eta)}{\sigma \eta} \cdot \ee\left[ 1 + \sup_{f \in \cF} \omega(f) \right] + \frac{2 L n \Delta}{\sigma}},
    \end{align}
    where we applied Lemma \ref{lem:blocketal} on the event \eqref{eq:normcomparison}.  Applying \citet[Lemma 36]{block2022smoothed}, we see that with probability at least $1 - \delta$, it holds that we may take
    \begin{align}
        \Delta \leq \BigOh{\frac{\ee\left[ \sup_{f \in \cF} \omega(f) \right]}{\eta m^{3/2}} + \frac{\sqrt{\log\left( \frac 1\delta \right)}}{m}}.
    \end{align}
    Applying \citet[Theorem 3.5.13]{gine2021mathematical}, we see that
    \begin{align}
        \ee\left[ \sup_{f \in \cF} \omega(f) \right] \leq \BigOh{\eta \sqrt{m \cdot \log\left( \cN_{[]}(\cF, \mu, \epsilon) \right)} + \eta m \epsilon}
    \end{align}
    for all $\epsilon > 0$.  Thus, setting
    \begin{align}
        m = \sqrt T, && \eta = \frac 1{\sigma} \cdot \sqrt{\frac{T L^3 n^3}{m}}, && n = \frac{T^{1/5} \sigma^{2/5}}{L^{3/5}} \cdot \log^{2/5}\left( \cN_{[]}\left( \cF, \mu, \frac{\sigma}{L T} \right)\right)
    \end{align}
    yields
    \begin{align}
        \ee\left[ \reg_T \right] \leq \BigOhTil{\frac{T^{4/5} L^{3/5}}{\sigma^{2/5}} \cdot \log^{3/5}\left( \cN_{[]}\left( \cF, \mu, \frac{\sigma}{L T} \right)\right)}
    \end{align}
    with the same number of oracle calls.  Thus, in particular, in order to achieve average regret at most $\epsilon$, it suffices to call $\ermoracle$
    \begin{align}
        \BigOhTil{\frac{\epsilon^{-4} L^{3/5}}{\sigma^{2/5}} \cdot \log^{3/5}\left( \cN_{[]}\left( \cF, \mu, \frac{\sigma}{L T} \right)\right)}
    \end{align}
    times.
\end{proof}
\section{Proof of Theorem \ref{thm:lazyftplexponential}}\label{app:ftpl}

In this section, we prove a more general version of \Cref{thm:lazyftplexponential}.  Recall that for fixed $n$ and $1 \leq \tau \leq T / n$, we let
\begin{align}
    \cI_\tau &= \left\{ t | (\tau - 1) n \leq t \leq \tau n \right\} &  \elltil_\tau(\theta) &= \sum_{t \in \cI_\tau} \ell(\theta, x_t) \\
    L_t(\theta) &= \sum_{s = 1}^t \ell(\theta, x_s) & \Ltil_\tau(\theta) &= \sum_{\tau' = 1}^\tau \elltil_{\tau'}(\theta).
\end{align}
Further, we suppose that $\thetatil_\tau$ is chosen such that for some real-valued function $\gamma: \rr^d \to \rr$, it holds that
\begin{align}\label{eq:lazyftplinstantiation}
    \Ltil_{\tau-1}(\thetatil_\tau) -  \eta\inprod{\xi}{\thetatil_\tau} \leq \gamma(\eta \xi) +\inf_{\theta \in \Theta} \Ltil_{\tau - 1}(\theta) - \eta \inprod{\xi}{\theta}.
\end{align}
We will assume that $\xi \sim \expo(1)$ is a random vector in $\rr^d$ whose coordinates are independently drawn according to a standard exponential distribution.  For fixed $\xi$, let $\thetatil_\tau(\xi)$ denote some $\thetatil_\tau$ satisfying \eqref{eq:lazyftplinstantiation}.  We prove the following result:
\begin{theorem}\label{thm:lazyftplexponentialfull}
    Suppose that we are in the constrained online learning setting of Proposition \ref{prop:lazyftpl} with $\Theta \subset \rr^d$ such that $\sup_{\theta, \theta' \in \Theta} \norm{\theta - \theta'}_1 = D < \infty$.  Suppose further that the $\cZ$-parameterized pseudo-metric $\rho$ satisfies the pseudo-isometry property  of Definition \ref{def:pseudoisometry} with respect to $\ell^1$ on $\rr^d$ and that $\sup_{\nu \in \cM} \ee_\nu\left[ \ell(\theta, z) - \ell(\theta',z) \right] \leq \sup_{\nu \in \cM} \ee_\nu\left[ \rho(\theta, \theta', z) \right]$.  If the learner plays Algorithm \ref{alg:lazyftplexponential} and $\eta = \Omega(n^2)$ (with the exact relation given in \eqref{eq:netatradeoff}), then the expected regret is bounded:
    \begin{align}
        \ee\left[ \reg_T \right] &\leq \eta D d +  2+ + \frac{3 D T}{n} \log\left( T \bracknum(\Theta, \rho, 1/T) \right) \\
        &+ 8 T \alpha \cdot d^{\frac{\beta}{2 - \beta}} \left( \frac D\eta \left( 4\ee\left[ \gamma(\xi) \right] + 8 \epsilon n + \delta + 3 D \log\left( \frac{\bracknum(\Theta, \rho, \epsilon)}{\delta} \right) \right)\right)^{\frac{\beta}{4 - 2\beta}},.
    \end{align}
\end{theorem}
Note that \Cref{thm:lazyftplexponential} follows immediately by considering the special case $\gamma(\xi) = 0$.

The proof of \Cref{thm:lazyftplexponentialfull} proceeds by first appealing to Lemma \ref{lem:btl} and then bounding the stability term.  The control of the stability term is broken into two parts: in the first, we apply the techniques of \citet{agarwal2019learning,suggala2020online} to show that if the stability term is small, then $\norm{\thetatil_\tau - \thetatil_{\tau+1}}_1$ is small in expectation; in the second, we apply the pseudo-isometry assumption along with control of the generalized brackets to conclude the proof using a self-bounding approach.

\subsection{Bounding the Stability Term}

In this section, we apply the techniques of \citet{suggala2020online} to control the expected stability of $\thetatil_\tau$ in $\norm{\cdot}_1$.  We have the following key lemma:
\begin{lemma}\label{lem:stability1}
    Suppose that $\theta_\tau \in \rr^d$ is chosen according to \eqref{eq:lazyftplinstantiation}.  Suppose further that the $\ell^\infty$ diameter of $\Theta$ is bounded above by $D$.  Then it holds that 
    \begin{align}
        \ee\left[ \norm{\theta_\tau - \theta_{\tau + 1}}_1 \right] \leq   d \cdot \sqrt{\frac{D}{\eta} \cdot \ee\left[ 4 \gamma(\xi) + \abs{\elltil_\tau(\thetatil_\tau) - \elltil_\tau(\thetatil_{\tau+1})} \right]}.
    \end{align}
\end{lemma}
To prove the result, we require minor modifications of the key monotonicity lemmas from \citet{suggala2020online}, where we apply their techniques without carrying a Lipschitz assumption on the losses.  First, we have:
\begin{lemma}\label{lem:monotonicity1}
    Suppose that $\xi, \xi' \in \rr^d$ with $\theta = \thetatil_\tau(\xi)$ and $\theta' = \thetatil_\tau(\xi')$ for some fixed $\tau$, as in \eqref{eq:lazyftplinstantiation}.  Then the following inequality holds:
    \begin{align}\label{eq:monotonicity1}
        \eta \cdot\inprod{\xi' - \xi}{\theta' - \theta} \geq - \left( \gamma(\xi) + \gamma(\xi') \right).
    \end{align}
\end{lemma}
\begin{proof}
    We compute:
    \begin{align}
        \Ltil_\tau(\theta) - \eta \inprod{\xi}{\theta} &\leq \Ltil_\tau(\theta') - \eta \inprod{\xi}{\theta'} + \gamma(\xi) \\
        &= \Ltil_\tau(\theta') - \eta \inprod{\xi'}{\theta'} + \eta \inprod{\xi' - \xi}{\theta'} + \gamma(\xi) \\
        &\leq \Ltil_\tau(\theta) - \eta \inprod{\xi'}{\theta} + \eta \inprod{\xi' - \xi}{\theta'}+ \gamma(\xi) + \gamma(\xi').
    \end{align}
    The result follows.
\end{proof}
The second necessary result is the analogue of \citet[Lemma 6]{suggala2020online}:
\begin{lemma}\label{lem:monotonicity2}
    Suppose that $\xi,\xi' \in \rr^d$ with $\thetatil_\tau = \thetatil_\tau(\xi)$, $\thetatil_{\tau'} = \thetatil_\tau(\xi')$ and $\thetatil_{\tau+1}, \thetatil_{\tau+1}'$ defined similarly for some fixed $\tau$.  Then the following inequality holds:
    \begin{align}
        \min\left( \inprod{\thetatil_\tau'}{\xi' - \xi}, \inprod{\thetatil_{\tau+1}'}{\xi' - \xi} \right) &\geq \max\left( \inprod{\thetatil_\tau}{\xi' - \xi}, \inprod{\thetatil_{\tau+1}}{\xi' - \xi} \right) \\
        &- \frac{2\left( \gamma(\xi) + \gamma(\xi') \right) + \abs{\elltil_\tau(\thetatil_\tau) - \elltil_\tau(\thetatil_{\tau + 1})}}\eta.
    \end{align}
\end{lemma}
\begin{proof}
    By construction, we compute:
    \begin{align}
        \Ltil_{\tau}(\thetatil_\tau) - \eta \inprod{\xi}{\thetatil_\tau} &= \Ltil_{\tau - 1}(\thetatil_\tau) - \eta \inprod{\xi}{\thetatil_\tau} + \elltil_\tau(\thetatil_\tau) \\
        &\leq \Ltil_{\tau - 1}(\thetatil_{\tau + 1}) - \eta \inprod{\xi}{\thetatil_{\tau + 1}} + \elltil_\tau(\thetatil_\tau) + \gamma(\xi) \\
        &= \Ltil_\tau(\thetatil_{\tau + 1}) - \eta \inprod{\xi}{\thetatil_{\tau + 1}} + \gamma(\xi) + \elltil_\tau(\thetatil_\tau) - \elltil_\tau(\thetatil_{\tau + 1}).
    \end{align}
    Again by construction, we have:
    \begin{align}
        \Ltil_\tau(\thetatil_\tau) - \eta \inprod{\xi}{\thetatil_\tau} &= \Ltil_\tau(\thetatil_\tau) - \eta \inprod{\xi'}{\thetatil_\tau} + \eta \inprod{\xi' - \xi}{\thetatil_\tau} \\
        &\geq \Ltil_\tau(\thetatil_{\tau+1}') - \eta \inprod{\xi'}{\thetatil_{\tau + 1}'} + \eta \inprod{\xi' - \xi}{\thetatil_\tau} - \gamma(\xi') \\
        &= \Ltil_\tau(\thetatil_{\tau +1}') - \eta \inprod{\xi}{\thetatil_{\tau + 1}'}  + \eta \inprod{\xi' - \xi}{\thetatil_\tau - \thetatil_{\tau +1}'} - \gamma(\xi') \\
        &\geq \Ltil_\tau(\thetatil_{\tau + 1}) - \eta \inprod{\xi}{\thetatil_{\tau + 1}}+ \eta \inprod{\xi' - \xi}{\thetatil_\tau - \thetatil_{\tau +1}'} - \gamma(\xi').
    \end{align}
    Combining the two preceding displays leads to the following inequality:
    \begin{align}
        \eta\inprod{\xi' - \xi}{\thetatil_\tau - \thetatil_{\tau + 1}'} \geq - 2\left( \gamma(\xi) + \gamma(\xi') \right) - \abs{\elltil_\tau(\thetatil_\tau) - \elltil_\tau(\thetatil_{\tau +1})}.
    \end{align}
    An identical argument yields
    \begin{align}
        \eta\inprod{\xi' - \xi}{\thetatil_\tau' - \thetatil_{\tau + 1}} \geq - 2\left( \gamma(\xi) + \gamma(\xi') \right) - \abs{\elltil_\tau(\thetatil_\tau) - \elltil_\tau(\thetatil_{\tau +1})}.
    \end{align}
    Applying Lemma \ref{lem:monotonicity1} gives
    \begin{align}
        \eta\inprod{\xi' - \xi}{\thetatil_\tau' - \thetatil_\tau} &\geq - \left( \gamma(\xi) + \gamma(\xi') \right) \\
        \eta\inprod{\xi' - \xi}{\thetatil_{\tau + 1}' - \thetatil_{\tau + 1}} &\geq - \left( \gamma(\xi) + \gamma(\xi') \right).
    \end{align}
    Combining the inequalities concludes the proof.
\end{proof}
We are now ready to prove the stability bound:
\begin{proof}[Proof of Lemma \ref{lem:stability1}]
    For some fixed $\tau$ and $\xi$, for all $1 \leq i \leq d$, let
    \begin{align}
        \thetamax = \max(\thetatil_{\tau,i}, \thetatil_{\tau+1,i}), && \thetamin = \min(\thetatil_{\tau,i}, \thetatil_{\tau+1,i}),
    \end{align}
    where $\thetatil_{\tau, i}$ denotes the $i^{th}$ coordinate of $\thetatil_\tau$.  Observe that $\abs{\thetatil_{\tau,i} - \thetatil_{\tau+1,i}} = \thetamax - \thetamin$.  Suppose that $\xi \sim \expo(1)$ and let $\xi' = \xi + k \be_i$.  Then, using the memoryless property of the exponential distribution and denoting
    \begin{align}
        \ee_{-i}[\cdot] = \ee\left[ \cdot | \xi_1, \dots, \xi_{i-1}, \xi_{i+1}, \dots, \xi_d \right],
    \end{align}
    we have
    \begin{align}
        \ee_{-i}\left[ \thetamin \right] &= \pp\left( \xi_i \leq k \right) \cdot \ee_{-i}\left[ \thetamin | \xi_i \leq k \right] + \pp\left( \xi_i >k \right) \cdot \ee_{-i}\left[ \thetamin | \xi_i > k \right] \\
        &\geq \left( 1 - e^{-k} \right) \left( \ee_{-i}\left[ \thetamax \right] - D\right) + e^{-k} \cdot \ee_{-i}\left[ \thetamin | \xi_i > k \right]\\
        &= \left( 1 - e^{-k} \right) \left( \ee_{-i}\left[ \thetamax \right] - D\right) + e^{-k} \cdot\ee\left[ \thetamin' \right]
    \end{align}
    where $\thetamin' = \thetamin(\xi')$.  The inequality follows from the assumption on the diameter of $\Theta$ and the second equality follows from the memoryless property of the exponential distribution.  Applying Lemma \ref{lem:monotonicity2} and observing that $\inprod{\thetatil_\tau}{\xi' - \xi} = k \thetatil_{\tau,i}$, we see that
    \begin{align}
        \ee_{-i}\left[ \thetamin' \right] &\geq \ee_{-i}\left[ \thetamax \right] - \ee_{-i}\left[\frac{2 (\gamma(\xi) + \gamma(\xi')) + \abs{\elltil_\tau(\thetatil_\tau) - \elltil_\tau(\thetatil_{\tau+1})}}{\eta k}\right].
    \end{align}
    Thus, combining the previous displays tells us that
    \begin{align}
        \ee_{-i}\left[ \thetamin \right] \geq \ee_{-i}[\thetamax] - \left( 1 - e^{-k} \right)D - \ee_{-i}\left[\frac{2 (\gamma(\xi) + \gamma(\xi')) + \abs{\elltil_\tau(\thetatil_\tau) - \elltil_\tau(\thetatil_{\tau+1})}}{\eta k}\right].
    \end{align}
    Then,
    \begin{align}
        \ee_{-i}\left[ \abs{\thetatil_{\tau,i} - \thetatil_{\tau+1,i}} \right] &= \ee_{-i}\left[ \thetamax - \thetamin \right] \\
        &\leq \left( 1 - e^{-k} \right)D + \ee_{-i}\left[\frac{2 (\gamma(\xi) + \gamma(\xi')) + \abs{\elltil_\tau(\thetatil_\tau) - \elltil_\tau(\thetatil_{\tau+1})}}{\eta k}\right] \\
        &\leq k D + \ee_{-i}\left[\frac{2 (\gamma(\xi) + \gamma(\xi')) + \abs{\elltil_\tau(\thetatil_\tau) - \elltil_\tau(\thetatil_{\tau+1})}}{\eta k}\right].
    \end{align}
    Summing over $1 \leq i \leq d$ and minimizing over $k$ concludes the proof.
\end{proof}

\subsection{Concluding the Proof}

We will apply the Be-the-Leader lemma; to do this, we need to bound the perturbation term and the stability terms.  For the first, we have the following result:
\begin{lemma}\label{lem:supremumofexpoprocess}
    Suppose that $\Theta \subset \rr^d$ such that $\sup_{\theta} \norm{\theta}_\infty \leq D$.  Then
    \begin{align}
        \ee\left[ \sup_{\theta \in \Theta} \inprod{\theta}{\xi}  \right] \leq Dd,
    \end{align}
    where $\xi \sim \expo(1)$.
\end{lemma}
\begin{proof}
    Observe that
    \begin{align}
        \ee\left[ \sup_{\theta \in \Theta} \inprod{\theta}{\xi} \right] &\leq \ee\left[ \sup_{\theta \in \Theta} \norm{\theta}_\infty \norm{\xi}_1 \right] \leq D \ee\left[ \norm{\xi}_1 \right] = D d.
    \end{align}
\end{proof}
We are now prepared to conclude the prove the main result:

\begin{proof}[Proof of \Cref{thm:lazyftplexponentialfull}]
    By Lemma \ref{lem:btl}, it suffices to bound the perturbation term and the stability terms independently.  To bound the stability terms, note that by the assumption of Lipschitzness with respect to $\rho$, we have
    \begin{align}
        \ee\left[ \elltil_\tau(\thetatil_\tau) - \elltil_\tau(\thetatil_{\tau + 1}) \right] &= \ee\left[ \sum_{t \in \cI_\tau} \ell(\thetatil_\tau, x_t) - \ell(\thetatil_{\tau+1}, x_t) \right] \\
        &\leq \ee\left[ \sum_{t \in \cI_\tau} \rho(\thetatil_\tau, \thetatil_{\tau+1},x_t) \right].
    \end{align}
    We now apply Proposition \ref{prop:singlestepbracketing} and observe that for all $\epsilon, \delta > 0$,
    \begin{align}
        \ee\left[ \sum_{t \in \cI_\tau} \rho(\thetatil_\tau, \thetatil_{\tau+1},x_t) \right] &\leq \ee\left[4 n \sup_{\nu \in \cM}\left[ \rho(\thetatil_\tau, \thetatil_{\tau + 1}, x) \right] + 8 \epsilon n + \delta + 3 D \log \left( \frac{4 \bracknum(\Theta, \rho, \epsilon)}{\delta} \right)\right] \\
        &\leq 4n \left(\alpha \ee\left[ \norm{\thetatil_\tau - \thetatil_{\tau + 1}}_1^\gamma \right]  + \beta\right) + 8 \epsilon n + \delta + 3 D \log\left( \frac{\bracknum(\Theta, \rho, \epsilon)}{\delta} \right) \\
        &\leq 4n \left(\alpha \left(\ee\left[ \norm{\thetatil_\tau - \thetatil_{\tau + 1}}_1\right]\right)^\beta\right) + 8 \epsilon n + \delta + 3 D \log\left( \frac{\bracknum(\Theta, \rho, \epsilon)}{\delta} \right) \label{eq:elltilupperbound},
    \end{align}
    where the second inequality follows from the pseudo-isometry property and the last inequality follows by Jensen's and the fact that $\gamma \leq 1$.  By Lemma \ref{lem:stability1}, we have
    \begin{align}
        \ee\left[ \norm{\thetatil_\tau - \thetatil_{\tau + 1}}_1 \right] &\leq d \sqrt{\frac{D}{\eta} \cdot \ee\left[ 4 \gamma(\xi) + \abs{\elltil_\tau(\thetatil_\tau)  - \elltil_\tau(\thetatil_{\tau + 1})} \right]} \\
        &\leq d \sqrt{\frac D\eta \left( 4\ee\left[ \gamma(\xi) \right] + 8 \epsilon n + \delta + 3 D \log\left( \frac{\bracknum(\Theta, \rho, \epsilon)}{\delta} \right) \right)} \\
        &+ d \cdot \sqrt{\frac{2 n D \alpha}{\eta} } \cdot \ee\left[ \norm{\thetatil_\tau - \thetatil_{\tau + 1}}_1 \right]^{\frac{\beta}{2}},
    \end{align}
    where the second inequality follows by plugging in the preceding display to Lemma \ref{lem:stability1} and applying subadditivity of the square root.  Rearranging tells us that
    \begin{align}
        \ee&\left[ \norm{\thetatil_\tau - \thetatil_{\tau + 1}}_1 \right] \\
        &\leq  \max\left(d^{\frac{\beta}{2 - \beta}} \left( \frac D\eta \left( 4\ee\left[ \gamma(\xi) \right] + 8 \epsilon n + \delta + 3 D \log\left( \frac{\bracknum(\Theta, \rho, \epsilon)}{\delta} \right) \right)\right)^{\frac{\beta}{4 - 2\beta}}, \left( \frac{2 d^2 D n \alpha }{\eta} \right)^{\frac{1}{2 - \beta}}   \right).
    \end{align}
    Plugging this into \eqref{eq:elltilupperbound} tells us that
    \begin{align}
        &\ee\left[ \elltil_\tau(\thetatil_\tau) - \elltil_\tau(\thetatil_{\tau+1}) \right] \leq 8 \epsilon n + \delta + 3 D \log\left( \frac{\bracknum(\Theta, \rho, \epsilon)}{\delta} \right)   \\
        &+ 8 n \alpha \cdot \max\left(d^{\frac{\beta}{2 - \beta}} \left( \frac D\eta \left( 4\ee\left[ \gamma(\xi) \right] + 8 \epsilon n + \delta + 3 D \log\left( \frac{\bracknum(\Theta, \rho, \epsilon)}{\delta} \right) \right)\right)^{\frac{\beta}{4 - 2\beta}}, \left( \frac{2 d^2 D n \alpha }{\eta} \right)^{\frac{1}{2 - \beta}}   \right).
    \end{align} 
    Summing over $\tau$, we see that the stability term in Lemma \ref{lem:btl} is bounded above by
    \begin{align}
        & 8 \epsilon T + \delta \frac{T}{n} + \frac{3 D T}{n} \log\left( \frac{\bracknum(\Theta, \rho, \epsilon)}{\delta} \right) \\
        &+ 8 T \alpha \cdot \max\left(d^{\frac{\beta}{2 - \beta}} \left( \frac D\eta \left( 4\ee\left[ \gamma(\xi) \right] + 8 \epsilon n + \delta + 3 D \log\left( \frac{\bracknum(\Theta, \rho, \epsilon)}{\delta} \right) \right)\right)^{\frac{\beta}{4 - 2\beta}}, \left( \frac{2 d^2 D n \alpha }{\eta} \right)^{\frac{1}{2 - \beta}}   \right).
    \end{align}
    Applying Lemma \ref{lem:btl} and \ref{lem:supremumofexpoprocess} tells us that the expected regret, then, is bounded by
    \begin{align}
        & \eta D d +  8 \epsilon T + \delta \frac{T}{n} + \frac{3 D T}{n} \log\left( \frac{\bracknum(\Theta, \rho, \epsilon)}{\delta} \right)  \\
        &+ 8 T \alpha \cdot \max\left(d^{\frac{\beta}{2 - \beta}} \left( \frac D\eta \left( 4\ee\left[ \gamma(\xi) \right] + 8 \epsilon n + \delta + 3 D \log\left( \frac{\bracknum(\Theta, \rho, \epsilon)}{\delta} \right) \right)\right)^{\frac{\beta}{4 - 2\beta}}, \left( \frac{2 d^2 D n \alpha }{\eta} \right)^{\frac{1}{2 - \beta}}   \right).
    \end{align}
    If we set $\delta = \epsilon = \frac{1}{T}$, note that as long as 
    \begin{align}\label{eq:netatradeoff}
        \eta \geq \frac{d^{4 - 2 \beta} \cdot D^{2 - \beta} \cdot d^{4 - 2\beta} \cdot \alpha^2}{\left( 4 \ee\left[ \gamma(\xi) \right] + 3 D \log\left( T \cdot \bracknum(\Theta, \rho, 1/T) \right)  \right)^{2 \beta}} \cdot n^2,
    \end{align}
    we have that the first argument to the maximum dominates the second, concluding the proof.
\end{proof}
\section{Proofs related to Piecewise Continuous Functions with Generalized Affine Boundaries}

In this section, we provide a detailed proof of \Cref{thm:pwabracketing}.  We then state and prove a similar result, replacing $\ellbar$ with the $\ell$ from \eqref{eq:piecewisecontinuous}, assuming an additional margin condition on the boundaries.  The latter is included both for increased generality and for its application to the multi-step planning problem of \Cref{sec:planning}.

\subsection{Proof of Theorem \ref{thm:pwabracketing}}\label{app:pwabracketing}

In this section we prove \Cref{thm:pwabracketing}.  We begin with the key step, showing that $\pp\left( \kbar_\phi(\thetad, \bz) \neq \kbar_\phi(\thetad', \bz) \right) \lesssim \norm{\thetad - \thetad'}_1$ if $\bz$ comes from a $\sigdir$-directionally smooth distribution.  We then apply this result both to control the pseudo-isometry constant and to bound the generalized bracketing numbers.  

We begin with the following lemma:
\begin{lemma}\label{lem:kbarcontinuity}
    Suppose that $\Thetad$, $\phibar$, and $\kbar_{\phibar}$ are defined as in \Cref{thm:pwabracketing} and suppose $\bz$ is chosen from a $\sigdir$-directionally smooth distribution such that $\norm{\bz}_\infty \leq B$ almost surely.  Then,
    \begin{align}
        \pp\left( \kbar_{\phibar}(\thetad, \bz) \neq \kbar_{\phibar}(\thetad', \bz) \right) \leq \frac{AB}{a\sigdir} \cdot \norm{\thetad - \thetad'}_1.
    \end{align}
\end{lemma}
\begin{proof}
    We begin by observing that
    \begin{align}
        \pp\left( \kbar_{\phibar}(\thetad, \bz) \neq \kbar_{\phibar}(\thetad', \bz) \right) &= \pp\left( \argmax_{k \in [K]} \sum_{k' \neq k} \I\left[ \phibar(\thetad, k, k', z) \geq 0 \right] \neq \argmax_{k \in [K]} \sum_{k' \neq k} \I\left[ \phibar(\thetad', k, k', z) \geq 0 \right] \right) \\
        &\leq \pp\left( \bigcup_{k, k' \in [K]} \left\{  \phibar(\thetad, k, k', z) \geq 0 > \phibar(\thetad', k, k', z) \right\} \right) \\
        &= \pp\left( \bigcup_{k, k' \in [K]} \left\{ \psi\left(\inprod{\bw_{kk'}}{\bz}\right) \geq 0 > \psi\left(\inprod{\bw_{kk'}'}{(\bz, 1)}\right) \right\}  \right) \\
        &\leq \sum_{k,k' \in [K]} \pp\left( \psi\left(\inprod{\bw_{kk'}}{(\bz, 1)}\right) \geq 0 > \psi\left(\inprod{\bw_{kk'}'}{(\bz, 1)}\right) \right), \label{eq:kbarunionbound}
    \end{align}
    where the first inequality follows from the fact that $\phibar$ is antisymmetric in $(k,k')$, the second equality follows from the construction of $\phibar$, and the last inequality follows from a union bound.  We now observe that for fixed $k,k' \in [K]$,
    \begin{align}
        \pp&\left( \psi(\inprod{\bw_{kk'}}{(\bz, 1)}) \geq 0 > \psi(\inprod{\bw_{kk'}'}{(\bz, 1)}) \right) \\
        &\leq \pp\left( \psi(\inprod{\bw_{kk'}}{(\bz, 1)}) \leq \abs{\psi(\inprod{\bw_{kk'}}{(\bz, 1)}) - \psi(\inprod{\bw_{kk'}'}{(\bz, 1)})} \right)\\
        &\leq \pp\left( \psi(\abs{\inprod{\bw_{kk'}}{(\bz, 1)}}) \leq A \abs{\inprod{\bw_{kk'} - \bw_{kk'}'}{(\bz, 1)}} \right) \\
        &\leq \pp\left( \psi(\abs{\inprod{\bw_{kk'}}{(\bz, 1)}} )\leq A B \norm{\bw_{kk'} - \bw_{kk'}'}_1 \right) \\
        &\leq \frac{AB}{a\sigdir} \cdot \norm{\bw_{kk'} - \bw_{kk'}'}_1,\label{eq:kbartriangleinequality}
    \end{align}
    where the first inequality follows from the triangle inequality, the second inequality follows from the assumption of $\psi$ being $A$-Lipschitz, the third inequality follows from H{\"o}lder's inequality and the fact that $\norm{\bz}_\infty \leq B$ almost surely, and the final inequality follows from the fact that $\bw_{kk'} \in \cS^{d}$ and the directional smoothness of $\bz$, along with \citet[Lemma 36]{block2022efficient}.  Plugging in to the first display and summing tells us that
    \begin{align}
        \pp\left( \kbar_{\phibar}(\thetad, \bz) \neq \kbar_{\phibar}(\thetad', \bz) \right) &\leq \sum_{k,k' \in [K] } \frac{AB}{a\sigdir} \cdot \norm{\bw_{kk'} - \bw_{kk'}'}_1 \\
        &= \frac{AB}{a\sigdir} \norm{\thetad - \thetad'}_1,
    \end{align}
    which concludes the proof.
\end{proof}
We now use Lemma \ref{lem:kbarcontinuity} to show that the pseudo-isometry property holds:
\begin{lemma}\label{lem:pwaisometry}
    Suppose that we are in the situation of \Cref{thm:pwabracketing} and $\cM$ is the class of $\sigdir$-directionally smooth distributions whose $\norm{\cdot}_\infty$ is almost surely bounded by $B > 0$.  Then
    \begin{align}
        \sup_{\nu \in \cM} \ee_{\nu}\left[ \rho(\theta, \theta', \bz) \right] \leq \frac{2 AB}{a\sigdir} \cdot \norm{\theta - \theta'}_1.
    \end{align}
\end{lemma}
\begin{proof}
    We compute:
    \begin{align}
        \ee_\nu\left[ \rho(\theta, \theta', \bz) \right] &= \ee_\nu\left[ 2 \cdot \I\left[ \kbar_{\phibar}(\thetad, \bz) \neq \kbar_{\phibar}(\thetad',\bz) \right] + \max_{k \in [K]} \norm{\thetac^{(k)} - \thetac^{'(k)}}_1\right] \\
        &\leq \frac{2AB}{a\sigdir} \norm{\thetad - \thetad'}_1 + \max_{k \in [K]} \norm{\thetac^{(k)} - \thetac^{'(k)}}_1 \\
        &\leq \frac{2A(B \vee 1) }{a\sigdir}\cdot \left( \norm{\thetad - \thetad'}_1 + \max_{k \in [K]} \norm{\thetac^{(k)} - \thetac^{'(k)}}_1\right) \\
        &\leq \frac{2A(B \vee 1) }{a\sigdir}\cdot \norm{\theta - \theta'}_1,
    \end{align}
    where the first inequality follows from linearity of expectation and Lemma \ref{lem:kbarcontinuity}.  The result follows.
\end{proof}
We now control the generalized bracketing number of $\Theta$:
\begin{lemma}\label{lem:pwabracketing}
    If we are in the situation of \Cref{thm:pwabracketing} then for any $\epsilon > 0$,
    \begin{align}
        \bracknum\left( \Theta, \rho, \epsilon \right) \leq \left( \frac{9 AK^2 B}{a\sigdir \epsilon} \right)^{K^2 d}.
    \end{align}
\end{lemma}
\begin{proof}
    Let $\cN = \left\{ \theta_i = (\thetac^i, \thetad^i) \right\}$ denote an $\epsilontil$-net of $\Theta$ with respect to $\ell_1$, where $\epsilontil = \frac{a\sigdir}{3 K^2A B} \cdot \epsilon$, and let
    \begin{align}
        \cB_i = \left\{ \theta \in \Theta | \norm{\theta - \theta_i}_1 \leq \epsilontil \right\}.
    \end{align}
    We claim that $\left\{ (\theta_i, \cB_i) \right\}$ forms a generalized $\epsilon$-bracket with respect to $\cM$, the class of $\sigdir$-directionally smooth distributions with $\ell_\infty$ norm bounded by $B$.  To see this, observe first that by the definition of an $\epsilontil$-net, it holds that the union of the $\cB_i$ covers $\Theta$.  Now, fix $\theta_i, \bz$ and observe that for $\theta \in \cB_i$, we have
    \begin{align}
        \rho(\theta, \theta_i, \bz) &= 2 \cdot \I\left[ \kbar_{\phibar}(\thetad, \bz) \neq \kbar_{\phibar}(\thetad^i, \bz) \right] + \max_{k \in [K]} \norm{\thetac - \thetac^i}_1 \\
        &\leq 2 \cdot \I\left[ \kbar_{\phibar}(\thetad, \bz) \neq \kbar_{\phibar}(\thetad^i, \bz) \right] + \norm{\thetac - \thetac^i}_1.
    \end{align}
    Now we compute:
    \begin{align}
        \pp&\left( \exists \theta \in \cB_i \text{ s.t. } \kbar_{\phibar}(\thetad, \bz) \neq \kbar_{\phibar}(\thetad^i, \bz) \right) \\
        &\leq \pp\left( \exists \theta \in \cB_i, k,k' \in [K] \text{ s.t. } \psi(\inprod{\bw_{kk'}}{(\bz, 1)}) \geq 0 > \psi(\inprod{\bw_{kk'}^i}{(\bz, 1)}) \right) \\
        &\leq \pp\left( \exists \theta \in \cB_i, k,k' \in [K] \text{ s.t. } \abs{\psi(\inprod{\bw_{kk'}^i}{(\bz, 1)})} \leq A B \norm{\bw_{kk'}^i - \bw_{kk'}}_1 \right) \\
        &\leq \pp\left( \exists \theta \in \cB_i, k,k' \in [K] \text{ s.t. } \abs{\psi(\inprod{\bw_{kk'}^i}{(\bz, 1)})} \leq A B \epsilontil \right) \\
        &\leq \sum_{k,k' \in [K]} \pp\left( \abs{\psi(\inprod{\bw_{kk'}^i}{(\bz, 1)})} \leq B \epsilontil \right) \\
        &\leq \frac{K^2AB}{a\sigdir} \epsilontil,
    \end{align}
    where the first inequality follows from the same reasoning as in \eqref{eq:kbarunionbound}, the second inequality follows from the same reasoning as in \eqref{eq:kbartriangleinequality}, the third inequality follows from the construction of $\cB_i$, the fourth inequality follows from a union bound, and the final inequality follows from the assumption of directional smoothness.  Thus, we note,
    \begin{align}
        \sup_{\nu \in \cM} \ee_\nu\left[ \sup_{\theta \in \cB_i} \rho(\theta, \theta_i, \bz) \right] &= \sup_{\nu \in \cM} \ee_\nu\left[ \sup_{\theta \in \cB_i} 2 \cdot \I\left[ \kbar_{\phibar}(\thetad, \bz) \neq \kbar_{\phibar}(\thetad^i, \bz) \right] + \norm{\thetac - \thetac^i}_1 \right] \\
        &\leq 2 \cdot \sup_{\nu \in \cM} \left\{\ee_\nu\left[ \sup_{\theta \in \cB_i} \I\left[ \kbar_{\phibar}(\thetad, \bz) \neq \kbar_{\phibar}(\thetad^i, \bz) \right] \right]\right\}  + \sup_{\theta \in \cB_i} \norm{\thetac - \thetac^i}_1 \\
        &\leq \left( \frac{2 K^2A B }{a\sigdir} + 1 \right) \epsilontil \\
        &\leq \epsilon.
    \end{align}
    Thus we have shown that $\left\{ (\theta_i, \cB_i) \right\}$ is a generalized $\epsilon$-bracket with respect to $\cM$.  It remains to bound the size.  To do this, note that by construction, it suffices to bound the size of an $\epsilontil$-cover with respect to $\ell_1$ on $\Theta$.  But note that $\Theta$ is contained in an $\ell_1$ ball of radius $D$ and thus a simple volume argument (see \citet[Section 4.2.1]{vershynin2018high} for example) tells us that, because $\Theta \subset \rr^{K d + K^2(d-1)}$, we may take
    \begin{align}
        \abs{\cN} \leq \left( \frac{3 D}{\epsilontil} \right)^{Kd + K^2 (d-1)} \leq \left( \frac{9 K^2 DAB}{a\sigdir \epsilon} \right)^{K^2 d}.
    \end{align}
    The result follows.
\end{proof}
The proof of Theorem \ref{thm:pwabracketing} follows from combining Lemmas \ref{lem:pwabracketing} and \ref{lem:pwaisometry}.

\subsection{Proof of Corollary \ref{cor:pwaftpl}}\label{app:pwaftplproof}
By applying \Cref{thm:lazyftplexponential,thm:pwabracketing}, it suffices to show that $\ellbar$ is Lipschitz with respect to the $\rho$ defined in \eqref{eq:rhodefinition}.  We observe, however, that
\begin{align}
    \ellbar&(\theta, \bz) - \ellbar(\theta', \bz) \\
    &= (\ellbar(\theta, \bz) - \ellbar(\theta', \bz)) \cdot \I\left[ \kbar_{\phibar}(\thetad, \bz) = \kbar_{\phibar}(\thetad', \bz) \right] + (\ellbar(\theta, \bz) - \ellbar(\theta', \bz))\cdot\I\left[ \kbar_{\phibar}(\thetad, \bz) \neq \kbar_{\phibar}(\thetad', \bz) \right] \\
    &\leq 2\cdot  \I\left[ \kbar_{\phibar}(\thetad, \bz) \neq \kbar_{\phibar}(\thetad', \bz) \right] + \max_{k \in [K]} g_k(\thetac, \bz) - g_k(\thetac', \bz) \\
    &\leq 2\cdot  \I\left[ \kbar_{\phibar}(\thetad, \bz) \neq \kbar_{\phibar}(\thetad', \bz) \right] + \max_{k \in [K]} \norm{\thetac^{(k)}- \thetac^{'(k)}}_1.
\end{align}
The result then follows by the definition of $\rho$.

\subsection{Replacing $\ellbar$ with $\ell$}\label{app:margin}

While the work in \Cref{app:pwabracketing} sufficed to prove \Cref{thm:pwabracketing}, for the sake of planning, we may wish to replace the loss function $\ellbar$ with the much simpler $\ell$ of \eqref{eq:piecewisecontinuous}.  In order to apply our techniques, however, we will require that $\thetad = (\bw_1, \dots, \bw_K)$ satisfies a certain margin condition.  The analogue of \Cref{app:pwabracketing} is thus:
\begin{theorem}\label{thm:pwabracketingmargin}
    Suppose that $\cZ \subset \rr^d$ and $\Theta$ is a subset of a Euclidean space of $\ell_1$ diameter bounded by $D$, and that
    \begin{align}
        \Thetad \subset \left\{ (\bw_1, \dots, \bw_K) \in (\cS^{d})^{\times K} | \min_{k \neq k' \in [K]} \norm{\bw_{k,\dhat} - \bw_{k',\dhat}}_2 \geq \gamma \right\},
    \end{align}
    where we denote by $\bw_k$ the coordinates of a given $\thetad \in \Thetad$ and let $\bw_{k,\dhat}$ denote the first $d$ coordinates of $\bw_k$.  Suppose further that $\phi(\thetad, k, \bz) = \psi(\inprod{\bw_k}{(\bz, 1)})$ for some link function $\psi$, as in \Cref{thm:pwabracketing}.  If $\cM$ consists of the class of $\sigdir$-directionally smooth distributions such that $\norm{\bz}_\infty \leq B$ almost surely for some $B \geq 1$, then with $\rho$ as in \eqref{eq:rhodefinition}, it holds that $\rho$ is a pseudo-metric satisfying the pseudo-isometry property with $\alpha = \frac{4 A B}{a \gamma \sigdir}$ and $\beta = 1$.  Furthermore, for all $\epsilon > 0$,
    \begin{align}
        \bracknum\left( \Theta, \rho, \epsilon \right) \leq \left( \frac{18 A K^2 B D}{a \gamma\sigdir\epsilon} \right)^{2 K (d + 1)}
    \end{align}
\end{theorem}
\begin{proof}
    The proof is essentially the same as that of \Cref{thm:pwabracketing} given in \Cref{app:pwabracketing}.  In fact, we simply need to prove a version of Lemma \ref{lem:kbarcontinuity} and the rest of the proof applies, \emph{mutatis mutandis}.  To see this, note that we may mimic the aforementioned proof by setting $\bw_{kk'} = \bw_k - \bw_{k'}$; this is almost the same as the previous scenario with the exception that we may now take $\norm{\bw_{kk'}} \neq 1$.  This causes a problem only in the application of directional smoothness; thus, suppose that $\bz$ is $\sigdir$ directionally smooth and observe that the chain of inequalities in \eqref{eq:kbarunionbound} remains valid.  Continuing, we see that for fixed $k,k' \in [K]$,
    \begin{align}
        \pp&\left( \psi(\inprod{\bw_{kk'}}{(\bz, 1)}) \geq 0 > \psi(\inprod{\bw_{kk'}'}{(\bz, 1)}) \right) \\
        &\leq \pp\left( \psi(\inprod{\bw_{kk'}}{(\bz, 1)}) \leq \abs{\psi(\inprod{\bw_{kk'}}{(\bz, 1)}) - \psi(\inprod{\bw_{kk'}'}{(\bz, 1)})} \right)\\
        &\leq \pp\left( \psi(\abs{\inprod{\bw_{kk'}}{(\bz, 1)}}) \leq A \abs{\inprod{\bw_{kk'} - \bw_{kk'}'}{(\bz, 1)}} \right) \\
        &\leq \pp\left( \psi(\abs{\inprod{\bw_{kk'}}{(\bz, 1)}} )\leq A B \norm{\bw_{kk'} - \bw_{kk'}'}_1 \right)\\
        &\leq  \pp\left( \psi(\abs{\inprod{\bw_{k} - \bw_{k'}}{(\bz, 1)}} ) \leq A B (\norm{\bw_k - \bw_{k}'}_1 + \norm{\bw_{k'} - \bw_{k'}'}_1) \right) \\
        &\leq \frac{AB}{a\sigdir \cdot \norm{\bw_{k,\dhat} - \bw_{k',\dhat}}_2} \cdot \left(\norm{\bw_k - \bw_{k}'}_1 + \norm{\bw_{k'} - \bw_{k'}'}_1\right) \\
        &\leq \frac{AB}{a\sigdir \gamma} \cdot \left(\norm{\bw_k - \bw_{k}'}_1 + \norm{\bw_{k'} - \bw_{k'}'}_1\right),
    \end{align}
    where the first four inequalities follow as in \eqref{eq:kbartriangleinequality}, the fifth inequality follows from the definition of $\bw_{kk'}$, the penultimate inequality follows as in the previous proof, and the last inequality follows from the margin assumption.  We then may apply the identical logic as in the proof of \Cref{thm:pwabracketing} going forward and the result holds, after channging the dimension of $\Thetad$ and adding a multiplicative factor of 2 to account for summing twice the differences $\norm{\bw_k - \bw_{k}'}_1$ above.
\end{proof}
Using the identical argument as in Corollary \ref{cor:pwaftpl}, we arrive at the following regret bound for Algorithm \ref{alg:lazyftplexponential} in the situation of \Cref{thm:pwabracketingmargin}:
\begin{corollary}\label{cor:pwaftplmargin}
    Suppose that $\ell$ is as in \eqref{eq:piecewisecontinuous} with $\phi$ and $\Theta$ as in \Cref{thm:pwabracketingmargin} with $B \geq 1$ and $\ell$ uniformly bounded in magnitude by $1$.  If we set $\eta = \BigOhTil{\left( T K A d D B (\gamma a \sigdir)^{-1} \right)^{2/3}}$ and $n = \sqrt{\eta}$, then Algorithm \ref{alg:lazyftplexponential} expereinces
    \begin{align}
        \ee\left[ \reg_T \right] \leq \BigOhTil{\left(\frac{T A K d B D}{a \gamma \sigdir}\right)^{2/3}}.
    \end{align}
    In particular, the oracle complexity of achieving average regret $\epsilon$ is $\BigOhTil{\frac{A K d D B}{\gamma a \sigdir \epsilon^2}}$.
\end{corollary}
\section{Proofs from Section \ref{sec:polies}}
In this appendix, we discuss the polynomially smooth assumption and provide examples of common distributions satisfying this requirement.  We also give a counter example that demonstrates that directional smoothness is \emph{not} sufficient to ensure polynomial smoothness.  We then prove \Cref{thm:polyregret}.

\subsection{Polynomial Smoothness}\label{app:polysmooth}
In this section, we discuss the notion of Polynomial smoothness found in Definition \ref{def:polysmooth}.  We begin by recalling the simple \citet[Example 3]{glazer2022anti}, which demonstrates that $\sigdir$-smoothness is strictly weaker than polynomial smoothness:
\begin{example}
    We show that there exists a class of distributions that is $\sigdir$-directionally smooth with $\sigdir$ decaying polynomially with dimension, but is $\sigpoly{2}$-polynomially smooth only for $\sigpoly{2}$ decaying exponentially with dimension.  Let $\nu_d$ denote the uniform measure on the unit Euclidean ball $\cB^d \subset \rr^d$.  We observe that by \citet[Example 1]{block2022efficient}, $\nu_d$ is $\sigdir$-directionally smooth with $\sigdir = \BigOmega{ \frac{1}{d} }$.  On the other hand, for the polynomial $f(x) =  \frac 1d \cdot  \norm{x}_2^2$, we see that $\coeff_2(f) = 1$, but concentration of measure (see \citet{vershynin2018high} for example) tells us that $\pp\left(\abs{f(x) - 1} \leq \epsilon  \right) \geq 1 -  e^{- \Omega(d) \epsilon}$.
\end{example}
We now consider what kinds of distributions are polynomially smooth.  The key tool in our arsenal is (a special case of ) the famous inequality of Carbery-Wright, which says:
\begin{theorem}[Theorem 8 from \citet{carbery2001distributional}]\label{thm:carberywright}
    If $\nu$ is a log-concave measure on $\rr^d$ and $f: \rr^d \to \rr$ is a degree $r$ polynomial, then for all $\epsilon > 0$, if $X \sim \nu$,
    \begin{align}
        \pp\left( \abs{f(X)} \leq \epsilon \right) \leq C r \frac{\epsilon^{\frac 1r}}{\ee\left[ f(X)^2 \right]^{\frac{1}{2r}}}.
    \end{align}
\end{theorem}
Thus, \Cref{thm:carberywright} tells us that if $\nu$ is log-concave and we can be assured that all polynomials $f$ with $\coeff_r(f) \geq 1$ have large second moment, then $\nu$ is $\sigpoly{r}$-polynomially smooth with $\sigpoly{r}$ depending nicely on the dimension.  Proving that polynomials with large coefficients indeed have large second moment is still an active area of research, but we provide as an example the following result, rephrased into our language:
\begin{theorem}[Corollary 4 from \citet{glazer2022anti}]
    Suppose that $\nu = \mu^{\otimes n}$ is a log-concave, isotropic product measure.  Then $\nu$ is $\sigpoly{r}$-polynomially smooth with $\sigpoly{r} \geq \BigOmega{\frac{1}{r}}$.
\end{theorem}
Note that the above result encompasses Gaussian measures and can be scaled as needed.  Further results in the direction of \citet{glazer2022anti} would translate directly into a wider class of measures known to be $\sigpoly{r}$-polynomially smooth.

\subsection{Proof of Theorem \ref{thm:polyregret}}\label{app:polies}
In this section, we provide a proof of \Cref{thm:polyregret} that follows the approach of Corollary \ref{cor:pwaftpl}.  While we do not repeat the argument, we observe in passing that replacing the tournament-style $\ellbar$ with the simpler function $\ell$ from \eqref{eq:piecewisecontinuous} and including a margin assumption allows for an analogue of \Cref{thm:pwabracketingmargin} in this setting.  To prove \Cref{thm:polyregret}, we begin by proving an analogue of \Cref{thm:pwabracketing}:
\begin{theorem}\label{thm:polybracketing}
    Suppose that $\cZ \subset \rr^d$ and that $\Theta$ is a subset of Euclidean space with $\ell_1$ diameter bounded by $D$.  Let $\Thetad$ parameterize the set of tuples of $\binom{K}{2}$ degree $r$ polynomials $(f_{\bw_{kk'}})$ on $\rr^d$ such that $\coeff_r(f_{\bw_{kk'}}) = 1$ for all $k \in [K]$.  If $\phibar(\thetad, k, k', \bz) = f_{\bw_{kk'}}(\bz)$ and $\cM$ is the class of $\sigpoly{r}$-polynomially smooth distributions such that $\norm{\bz}_\infty \leq B$ almost surely, then the $\rho$ defined in \eqref{eq:rhodefinition} is a pseudo-metric satisfying the pseudo-isometry property with $\alpha = \frac{2 B^r D}{\sigpoly{r}}$ and $\beta = \frac 1r$.  Furthermore, for all $\epsilon > 0$,
    \begin{align}
        \bracknum(\Theta, \rho, \epsilon) \leq \left( \frac{9 K^2 B}{\sigpoly{r} \epsilon} \right)^{K^2 r^2 d^r}.
    \end{align}
\end{theorem}
To prove the result, we need analogues of Lemmas \ref{lem:kbarcontinuity},\ref{lem:pwaisometry}, and \ref{lem:pwabracketing}.  We begin with proving the Lipschitzness in expectation of the first term of $\rho$:
\begin{lemma}\label{lem:kbarcontinuitypoly}
    Suppose that $\Thetad$, $\phibar$, and $\kbar_{\phibar}$ are as in \Cref{thm:polybracketing} and suppose that $\bz$ is chosen from a $\sigpoly{r}$-polynomially smooth distribution such that $\norm{\bz}_\infty \leq B$ almost surely for some $B \geq 1$.  Then,
    \begin{align}
        \pp\left( \kbar_{\phibar}(\thetad, \bz)  \neq \kbar_{\phibar}(\thetad', \bz)\right) \leq \frac{B^r K^{2 - \frac{4}{r}}}{\sigpoly{r}} \cdot \norm{\thetad - \thetad'}_1^{\frac 1r}.
    \end{align}
\end{lemma}
\begin{proof}
    By the same argument as in \eqref{eq:kbarunionbound}, we have that
    \begin{align}
        \pp\left( \kbar_{\phibar}(\thetad, \bz)  \neq \kbar_{\phibar}(\thetad', \bz)\right) &\leq \sum_{k,k' \in [K]} \pp\left( f_{\bw_{kk'}}(\bz) \geq 0 > f_{\bw_{kk'}'}(\bz) \right).
    \end{align}
    Observe that by the triangle inequality,
    \begin{align}
        \abs{f_{\bw_{kk'}}(\bz) - f_{\bw_{kk'}'}(\bz)} \leq B^r \cdot \norm{\bw_{kk'} - \bw_{kk'}'}_1.
    \end{align}
    Thus, applying the argument in \eqref{eq:kbartriangleinequality}, we have
    \begin{align}
        \pp\left( f_{\bw_{kk'}}(\bz) \geq 0 > f_{\bw_{kk'}'}(\bz) \right) &\leq \pp\left( \abs{f_{\bw_{kk'}}(\bz)} \leq \abs{\abs{f_{\bw_{kk'}}(\bz) - f_{\bw_{kk'}'}(\bz)}} \right) \\
        &\leq \pp\left(\abs{f_{\bw_{kk'}}(\bz)} \leq B^r\cdot  \norm{\bw_{kk'} - \bw_{kk'}'}_1  \right) \\
        &\leq \frac{B}{\sigpoly{r}} \cdot \norm{\bw_{kk'} - \bw_{kk'}'}_1^{\frac 1r},
    \end{align}
    where the last inequality follows from the definition of polynomial smoothness.  Applying H{\"o}lder's inequality and summing concludes the proof.
\end{proof}
Using this result, we can prove an analogue of Lemma \ref{lem:pwaisometry}:
\begin{lemma}\label{lem:polyisometry}
    Suppose that we are in the situation of \Cref{thm:polybracketing} and $\cM$ is the class of $\sigpoly{r}$-polynomially smooth distributions such that the infinity norms of samples are uniformly bounded almost surely by some $B \geq 1$.  If the $\ell_1$ diameter of $\Theta$ is bounded by $D$, then
    \begin{align}
        \sup_{\nu \in \cM} \ee_\nu\left[ \rho(\theta, \theta', \bz) \right] \leq \frac{2 B D}{\sigpoly{r}} \cdot \norm{\theta - \theta'}^{\frac 1r}.
    \end{align}
\end{lemma}
\begin{proof}
    We compute:
    \begin{align}
        \ee_\nu\left[ \rho(\theta, \theta', \bz) \right] &= \ee_\nu\left[ 2 \cdot \I\left[ \kbar_{\phibar}(\thetad, \bz) \neq \kbar_{\phibar}(\thetad', \bz) \right]+ \max_{k \in [K]} \norm{\thetac^{(k)} - \thetac^{'(k)}}_1 \right]  \\
        &\leq \frac{2 B}{\sigpoly{r}} \cdot \norm{\thetad - \thetad'}_1^{\frac 1r} + \norm{\thetac - \thetac'}_1 \\
        &\leq \frac{2 B D}{\sigpoly{r}} \cdot \norm{\theta - \theta'}^{\frac 1r},
    \end{align}
    where the second inequality follows from Lemma \ref{lem:kbarcontinuitypoly} and the last inequality follows by the assumption on the diameter.
\end{proof}
Finally, we require an analogue of Lemma \ref{lem:pwabracketing}:
\begin{lemma}\label{lem:polybracketing}
    If we are in the situation of \Cref{thm:polybracketing} then for any $\epsilon > 0$, it holds that
    \begin{align}
        \bracknum(\Theta, \rho, \epsilon) \leq \left( \frac{9 K^2 B}{\sigpoly{r} \epsilon} \right)^{K^2 r^2 d^r}.
    \end{align}
\end{lemma}
\begin{proof}
    We mimic the proof of Lemma \ref{lem:pwabracketing} but apply Lemma \ref{lem:kbarcontinuitypoly} instead of Lemma \ref{lem:kbarcontinuity}.  In particular, we suppose that $\cN = \left\{ \theta_i = \left( \thetac^i, \thetad^i \right) \right\}$ is an $\epsilontil$-net of $\Theta$ with respect to $\ell_1$, where $\epsilontil = \left( \frac{\sigpoly{r}}{3 K^2 B} \cdot \epsilon \right)^r$ and similarly let $\cB_i \subset \Theta$ denote the set of parameters within $\epsilontil$ of $\theta_i$ in $\ell_1$ norm.  We compute as in the proof of Lemma \ref{lem:pwabracketing} that
    \begin{align}
        \pp\left( \exists \theta \in \cB_i \text{ s.t. } \kbar_{\phibar}(\thetad, \bz) \neq \kbar_{\phibar}(\thetad^i, \bz) \right) &\leq \pp\left(\exists \theta \in \cB_i \text{ s.t. }  \abs{f_{\bw_{kk'}^i}(\bz)} \leq \abs{f_{\bw_{kk'}}(\bz) - f_{\bw_{kk'}^i}(\bz)}  \right) \\
        &\leq \pp\left( \exists \theta \in \cB_i \text{ s.t. } \abs{f_{\bw_{kk'}^i}(\bz)} \leq B^r \cdot \norm{\bw_{kk'} - \bw_{kk'}^i}_1 \right) \\
        &\leq \pp\left( \exists \theta \in \cB_i \text{ s.t. } \abs{f_{\bw_{kk'}^i}(\bz)} \leq B^r \cdot \epsilontil \right) \\
        &\leq \frac{K^2 B}{\sigpoly{r}} \cdot \epsilontil^{\frac 1r}.
    \end{align}
    Thus
    \begin{align}
        \sup_{\nu \in \cM} \ee_\nu\left[ \rho(\theta, \theta_i, \bz) \right] &\leq 2 \cdot \sup_{\nu \in \cM}\left\{ \ee_\nu\left[ \sup_{\theta \in \cB_i} \I\left[ \kbar_{\phibar}(\thetad, \bz) \neq \kbar_{\phibar}(\thetad^i, \bz) \right] \right] \right\} + \sup_{\theta \in \cB_i} \norm{\thetac - \thetac^i}_1 \\
        &\leq \frac{2 K^2 B}{\sigpoly{r}} \cdot \epsilontil^{\frac 1r} + \epsilontil \\
        &\leq \frac{3 K^2 B}{\sigpoly{r}} \cdot \epsilontil^{\frac 1r} \\
        &\leq \epsilon.
    \end{align}
    Thus $\cN$ is a generalized $\epsilon$-bracket with respect to $\cM$ and $\rho$.  We may bound the size of $\cN$ in the same way as in the proof of Lemma \ref{lem:pwabracketing}, after observing that $\bw_{kk'}$ lives in a space of dimension
    \begin{align}
         \sum_{i \leq r} \binom{d}{i}  \leq \left( \frac{e d}{r} \right)^{r}.
    \end{align}
\end{proof}
Combining Lemmas \ref{lem:polyisometry} and \ref{lem:polybracketing} concludes the proof of \Cref{thm:polybracketing}.  We are now ready to prove \Cref{thm:polyregret}:
\begin{proof}[Proof of \Cref{thm:polyregret}]
    We observe by the same logic as in the proof of Corollary \ref{cor:pwaftpl} that $\ellbar$ is Lipschitz with respect to $\rho$.  Thus we may apply \Cref{thm:lazyftplexponential} and \Cref{thm:polybracketing} to get that if Algorithm \ref{alg:lazyftplexponential} is played, then
    \begin{align}
        \ee\left[ \reg_T \right] \leq \BigOhTil{\eta + \frac{T}{n} K^2 r^2 d^r \log\left( \frac{1}{\sigpoly{r}} \right) + \frac{T B^r D}{\sigpoly{r}} \cdot \left( \frac{K^2 r^2 d^r }{\eta} \right)^{\frac{1}{4r - 2}} }.
    \end{align}
    Setting
    \begin{align}
        \eta = \BigThetaTil{\left( \frac{T K^2 r^2 d^r D B}{\sigpoly{r}} \right)^{\frac{4r - 2}{4r - 1}}} && n = \BigThetaTil{\left( \frac{T K^2 r^2 d^r D B}{\sigpoly{r}} \right)^{\frac{2r - 1}{4r - 1}}}
    \end{align}
    concludes the proof.
\end{proof}
\newcommand{\diam}{\mathrm{diam}}
\newcommand{\range}{\mathrm{range}}
\newcommand{\Lip}{\mathrm{Lip}}

\section{Proof of Theorem \ref{thm:planning}}\label{app:planning}
In this section, we state and prove a formal version of \Cref{thm:planning}. We recall that we are in the situation of \eqref{eq:dynamics} and that our aim is to minimize the regret with respect to the best plan $\bbaru_{1:H}$. Throughout, we let $\|\cdot\|_1$ denote the $\ell_1$ norm interpreted in the natural sense for concatenated vectors; e.g. $\|\bbaru_{1:H}\|_1 = \sum_{h=1}^H\|\bu_h\|_1$.  We begin by introducing a notation that will substantially simplify our presentation:

\begin{definition}\label{defn:io_map}  For a given sequence of modes $k_{1:H} \in [K]^H$, recall from \eqref{eq:dynamics} that the states evolve as
\begin{align}
&\tilde\bx_{t,h+1}(\theta ; k_{1:H}) = g_{t,h,k_h}(\tilde\bx_{t,h}(\theta;k_{1:H}), \bu_{t,h}(\theta)) + \boldeta_{t,h}, ~~ \tilde\bx_{t,1}(\theta ; k_{1:H}) = \bx_{t,1}, ~~\bu_{t,h}(\theta) = \bbaru_{h} + \bxi_{t,h},
\end{align}
where the difference between here and the situation in \eqref{eq:dynamics} is that here the mode sequence is given, whereas in \eqref{eq:dynamics} it was state and input dependent.  We define the function
\begin{align}
G_t(\theta;k_{1:H}) := \left(\tilde\bx_{t,1}(\theta ; k_{1:H}), \dots, \tilde\bx_{t,H}(\theta ; k_{1:H})\right) \in \cX^H \subset \R^{\dimx H},
\end{align}
which maps a plan and given mode sequence to the associated trajectory.
\end{definition}

We are now ready to state a formal version of \Cref{thm:planning}:
\begin{theorem}\label{thm:planningformal}
    For fixed planning horizon $H$, suppose that trajectories $\bx_{t,1:H}$ evolve as in \eqref{eq:dynamics}, where the learner chooses a plan $\theta \in \cK \subset \cU^{\times H}$ at each time $t$ and the adversary presents the tuple $\bz_t$ described in \Cref{sec:planning}.  Assume that for all $t \in [T]$ the following properties hold almost surely under the adversary's strategy $p_t$:
    \begin{enumerate}
        \item $\bx_{t,1} \mid \cF_t$ and $(\boldeta_{t,h},\bxi_{t,h}) \mid \cF_{t,h-1}$ are $\sigdir$-directionally smooth.
        \item For all mode sequences $k_{1:H} \in [K]^{H}$ and $\theta,\theta' \in \cU^{\times H}$, $\norm{G_t(\theta, k_{1:H})-G_t(\theta', k_{1:H})} \leq L\norm{\theta-\theta'}_1$, i.e., the functions $G_t$ are $L$-Lipschitz with respect to the $\ell_1$ norm.
        \item For all $h \in [H]$, $\sup_{\theta \in \cK} \norm{\bu_{t,h}(\theta)}_1 \vee \norm{\bx_{t,h}(\theta)}_1 \leq D$.
        \item For some $\gamma > 0$, it holds for all $h \in [H]$ that $\min_{k \neq k'} \norm{\bw_{t,h,k,\dhat} - \bw_{t,h,k',\dhat}}_2 \geq \gamma$, where we let $\bw_{\dhat}$ denote the first $d$ coordinates of the vector $\bw \in \rr^{d+1}$.
        \item For all $\bv_{1:H},\bv_{1:H}' \in \cV^H$ with $\|\bv_{1:H}\|_1 \vee \|\bv_{1:H}'\|_1 \le 2D$, we have that the loss functions $\lv_{t}$ are Lipschitz with respect to the $\ell_1$ norm and bounded, i.e.,  $|\lv_{t}(\bv_{1:H}) - \lv_t(\bv_{1:H}')| \le \|\bv_{1:H}-\bv_{1:H}\|_1$ and $|\lv_t(\bv_{1:H})| \le 1$.
    \end{enumerate}
    If the planner plays $\theta_t$ according to Algorithm \ref{alg:lazyftplexponential} with $\eta = d^{1/3} H^{5/3} K^{4/3} \left( \frac{T L D}{\gamma \sigdir} \right)^{2/3}$ and $n = \sqrt{\eta}$, then
    \begin{align}
        \ee\left[ \reg_T \right] \leq \BigOhTil{d^{1/3} H^{5/3} K^{4/3} \left( \frac{T L d}{\gamma \sigdir} \right)^{2/3}}.
    \end{align}
    Thus the oracle complexity of achieving average regret $\epsilon$ is $\BigOhTil{d^{1/3} D^{2/3} H^{5/3} K^{4/3} L^{2/3} (\gamma \sigdir)^{-2/3} \epsilon^{-2}}$.
\end{theorem}

\begin{remark}[Scaling of $L$ and $D$]\label{rem:L_scaling} Notice that the scaling of the parameters $L$ and $D$ depend on the $H$-fold compositions of the dynamic maps $g_{t,k,h}$.  Thus in the second assumption in the above theorem, requiring that the maps $G_t$ are $L$-Lipschitz, can na{\"i}vely allow $L$ to scale like $\mathrm{poly}(H)\cdot(\Lambda)^H$ in the worst case when we only assume that the functions $g_{t,k,h}$ are $\Lambda$-Lipschitz.  Similarly, if we only suppose that $g_{t,k,h}(\bv) \le c_2\|\bv\| + c_1$, the bound $D$ in the third and last assumption above could scale with $(c_2)^H$.  While these bounds are tight in general, the exponential dependencies can be mitigated with common stability assumptions, often imposed in control settings. For example, under incremental stability of the composed dynamics \citep{pfrommer2022tasil,angeli2000lyapunov}, $L$ would scale only polynomially in $H$.  Further notions of input-to-state stability such as those found in \citet{jadbabaie2001stability} result in polynomially-bounded $D$.  Thus in many practical settings of interest, the parameters $L,D$ scale only polynomially in all of the relevant problem parameters.  Note that in the popular Linear Quadratic Regulator framework, of which our setting is a vast generalization, these stability assumptions are standard \citep{hazan2022introduction}.
\end{remark}
As in the previous applications of Theorem \ref{thm:lazyftplexponential}, we will prove an analogue of Theorem \ref{thm:pwabracketing} where we introduce a pseudo-metric $\rho$ and prove that it satisfies pseudo-isometry and provides control of the generalized bracketing numbers.  We will then conclude by proving that our $\ell$ is Lipschitz with respect to $\rho$ and appealing to \Cref{thm:lazyftplexponential}.  For the sake of simplicity, we will drop the index $t$ temporarily and compare a given plan $\bbaru_{1:H}$ and its associated dynamics $\bx_{1:H}$ with an alternative plan $\bbaru_{1:H}'$ and its associated dynamics $\bx_{1:H}'$, where these dynamics share the noise sequences $\eta_{1:H}, \xi_{1:H}$.  We will abbreviate
\begin{align}\label{eq:rhoplanning}
    \rho(\theta, \theta'') = \rho(\theta, \theta', \boldeta_{1:H}, \bxi_{1:H}, \theta) = \norm{\theta - \theta'}_1 + \sum_{h = 1}^H \norm{\bx_{h} - \bx_{h}'}.
\end{align}
We will also abbreviate $k_h = k_h(\bv_h)$ and $k_h' = k_h(\bv_h')$.  We begin by proving the following lemma:
\begin{lemma}\label{lem:xdistupperbound}
    Consider the event
    \begin{align}
        \cA_h = \left\{ k_{h'} = k_{h'}' \text{ for all } h' < h \text{ and } k_h \neq k_h'  \right\}.
    \end{align}
    Then for all $h \in [H]$,
    \begin{align}
        \norm{\bx_h - \bx_h'}_1 \leq L \cdot \norm{\theta - \theta'}_1 + 2D \cdot \sum_{h' = 1}^{h-1} \I\left[ \cA_{h'} \right].
    \end{align}
\end{lemma}
\begin{proof}
    By \eqref{eq:dynamics}, we have
    \begin{align}
        \norm{\bx_h - \bx_{h}'}_1 &= \norm{g_{k_{h-1}}(\bv_{h-1}) - g_{k_{h-1}'}(\bv_{h-1}')}_1 \\
        &\leq \norm{g_{k_{h-1}}(\bv_{h-1}) - g_{k_{h-1}}(\bv_{h-1}')}_1 \cdot \I\left[ \bigcup_{h' \geq h} \cA_{h'} \right] + 2D \cdot \I\left[\bigcup_{h' < h} \cA_{h'}  \right] \\
        &\leq \norm{G(\theta, k_{1:H}) - G\left(\theta', k_{1:H}' \right)} + 2D \cdot \I\left[\bigcup_{h' < h} \cA_{h'}  \right] \\
        &\leq L \cdot \norm{\bbaru_{1:H} - \bbaru_{1:H}'}_1 + 2D \sum_{h' = 1}^{h-1} \I\left[ \cA_{h'} \right],
    \end{align}
    where the equality is by construction, the first inequality follows from the boundedness of $\cZ$, the second inequality follows from the fact that the definition of $\cA_{h'}$ for $h' \geq h$ implies that the modes $k_{h''}$ for $h'' < h$ are the same as $k_{h''}'$, and the last inequality follows from a union bound and the second condition in \Cref{thm:planningformal}.  The result follows.
\end{proof}
We are now ready to prove the pseudo-isometry property:
\begin{lemma}\label{lem:planningpseudoisometry}
    Let $\cM$ denote the class of distributions induced by the setup in \Cref{thm:planningformal} and let $\rho$ be as in \eqref{eq:rhoplanning}.  Then $\rho$ satisfies the pseudo-isometry property with respect to $\norm{\cdot}_1$ with $\alpha = \frac{6 D H^2K^2 L}{\gamma \sigdir}$ and $\beta = 1$, i.e.,
    \begin{align}
        \sup_{\nu \in \cM} \ee_\nu\left[ \rho(\theta, \theta') \right] \leq \frac{6 D H^2K^2 L}{\gamma \sigdir} \cdot \norm{\theta - \theta'}_1.
    \end{align}
\end{lemma}
\begin{proof}
    By \eqref{eq:rhoplanning} and Lemma \ref{lem:xdistupperbound}, it holds that
    \begin{align}
        \sup_{\nu \in \cM} \ee_\nu\left[ \rho(\theta, \theta') \right] &= \sup_{\nu \in \cM} \ee_\nu\left[ \norm{\theta - \theta'}_1 + \sum_{h = 1}^H \norm{\bx_{h} - \bx_{h}'} \right] \\
        &\leq 2 L H \cdot \norm{\theta - \theta'}_1 + 2 D H \cdot \sup_{\nu \in \cM} \sum_{h = 1}^H \pp_\nu( \cA_h ).
    \end{align}
    We now compute,
    \begin{align}
        \pp(\cA_h) &= \pp\left( k_{h'} = k_{h'}' \text{ for } h' < h \text{ and } k_h \neq k_{h}' \right) \\
        &=\pp\left( k_{h'} = k_{h'}' \text{ for } h' < h \text{ and } \argmax_{k \in [K]} \inprod{\bwstar{k,h}}{(\bv_h, 1)} \neq \argmax_{k' \in [K]} \inprod{\bwstar{k',h}}{(\bv_h', 1)} \right) \\
        &\leq \sum_{k \neq k' \in [K]} \pp\left(k_{h'} = k_{h'}' \text{ for } h' < h \text{ and } \inprod{\bwstar{k,h} - \bwstar{k',h}}{(\bv_h, 1)} \geq 0 >  \inprod{\bwstar{k,h} - \bwstar{k',h}}{(\bv_h', 1)}  \right)\label{eq:planningunion},
    \end{align}
    where the argument is similar to that in \eqref{eq:kbarunionbound}.  For fixed $k \neq k' \in [K]$, we then compute
    \begin{align}
        \pp&\left(k_{h'} = k_{h'}' \text{ for } h' < h \text{ and } \inprod{\bwstar{k,h} - \bwstar{k',h}}{(\bv_h, 1)} \geq 0 >  \inprod{\bwstar{k,h} - \bwstar{k',h}}{(\bv_h', 1)}  \right) \\
        &\leq \pp\left(k_{h'} = k_{h'}' \text{ for } h' < h \text{ and } \abs{\inprod{\bwstar{k,h} - \bwstar{k',h}}{(\bv_h, 1)}} \leq  \abs{\inprod{\bwstar{k,h} - \bwstar{k',h}}{(\bv_h - \bv_h', 0)}}  \right) \\
        &\leq \pp\left(k_{h'} = k_{h'}' \text{ for } h' < h \text{ and } \abs{\inprod{\bwstar{k,h} - \bwstar{k',h}}{(\bv_h, 1)}} \leq  2  \cdot \norm{\bv_h - \bv_h'}_1  \right), \label{eq:planningtriangle}
    \end{align}
    where the first inequality follows from the triangle inequality and the second from H{\"o}lder's inequality.  We now observe that
    \begin{align}
        \norm{\bv_h - \bv_h'}_1 &\leq \norm{\bu_h - \bu_h'}_1 + \norm{\bx_h - \bx_h'}_1,
    \end{align}
    and, furthermore, because the mode sequences $k_{1:h-1} = k_{1:h-1}'$, we have in this event that 
    \begin{align}
        \norm{\bx_h - \bx_h'}_1 &\leq \norm{G(\theta, k_{1:H}) - G(\theta', k_{1:H})}_1  \leq L \cdot \norm{\theta - \theta'}_1. \label{eq:planningzupper}
    \end{align}
    Thus, we have
    \begin{align}
        \pp&\left(k_{h'} = k_{h'}' \text{ for } h' < h \text{ and } \abs{\inprod{\bwstar{k,h} - \bwstar{k',h}}{(\bv_h, 1)}} \leq  2  \cdot \norm{\bv_h - \bv_h'}_1  \right) \\
        &\leq \pp\left(\abs{\inprod{\bwstar{k,h} - \bwstar{k',h}}{\bv_h}} \leq  3L  \cdot \norm{\theta - \theta'}_1  \right) \\
        &= \pp\left(\abs{\inprod{\frac{\bwstar{k,h} - \bwstar{k',h}}{\norm{\bwstar{k,h,\dhat} - \bwstar{k',h,\dhat}}_2}}{(\bv_h, 1)}} \leq  \frac{3L}{\norm{\bwstar{k,h,\dhat} - \bwstar{k',h,\dhat}}_2}  \cdot \norm{\theta - \theta'}_1  \right) \\
        &\leq \frac{3 L}{\sigdir\cdot \norm{\bwstar{k,h,\dhat} - \bwstar{k',h,\dhat}}_2 } \cdot \norm{\theta - \theta'}_1 \\
        &\leq \frac{3 L}{\sigdir \gamma} \cdot \norm{\theta - \theta'}_1,\label{eq:planningmargin}
    \end{align}
    where the first inequality follows from the preceding computation, the equality is trivial, the second inequality follows from the assmption of directional smoothness, and the last inequality follows from the margin assumption.  Plugging back in, the result follows.
\end{proof}
Finally, we prove that the bracketing numbers can be controlled:
\begin{lemma}\label{lem:planningbracketing}
    Let $\cM$ and $\rho$ be as in Lemma \ref{lem:planningpseudoisometry} and $\cK$ be as in \Cref{thm:planningformal}.  Then, for any $\epsilon > 0$, it holds that
    \begin{align}
        \bracknum\left( \cK,  \rho, \epsilon \right) \leq \left( \frac{36K^2 D H^2 L}{\gamma \sigdir \epsilon} \right)^{H (d + 1)}.
    \end{align}
\end{lemma}
\begin{proof}
    Let $\cN = \left\{ \theta^i \right\} \subset \cU^{\times H}$ denote an $\epsilontil$-net with respect to $\ell_1$, where $\epsilontil = \frac{\gamma \sigdir}{12 D H^2 L} \cdot \epsilon$.  As in Lemmas \ref{lem:pwabracketing} and \ref{lem:polybracketing}, we will show that if $\cB_i$ denotes the set of $\theta$ with distance at most $\epsilontil$ to $\theta^i$, then $\left\{ (\theta^i, \cB_i) \right\}$ forms a generalized $\epsilon$-bracket with respect to $\cM$.  The argument is essentially identical after we replace Lemmas \ref{lem:kbarcontinuity} and \ref{lem:kbarcontinuitypoly} with Lemma \ref{lem:xdistupperbound} and the argument in Lemma \ref{lem:planningpseudoisometry}.  In particular, for fixed $i$ and any $\nu \in \cM$, we see that
    \begin{align}
        \ee_\nu\left[ \sup_{\theta \in \cB_i} \rho(\theta, \theta^i) \right] &= \ee_\nu\left[ \sup_{\theta \in \cB_i}\norm{\theta - \theta^i}_1 + \sum_{h = 1}^H \norm{\bx_{h} - \bx_{h}^i}_1 \right] \\
        &\leq \epsilontil + \ee_\nu\left[ \sup_{\theta \in \cB_i} \sum_{h = 1}^H \norm{\bx_h - \bx_h^i}_1 \right],
    \end{align}
    where we let $\bx_{1:H}^i$ denote the dynamics evolved with $\theta^i$.  For the second term, we invoke Lemma \ref{lem:xdistupperbound} and compute:
    \begin{align}
        \ee_\nu\left[ \sup_{\theta \in \cB_i} \sum_{h = 1}^H \norm{\bx_h - \bx_h^i}_1 \right] &\leq \sum_{h = 1}^H \ee_\nu\left[ \sup_{\theta \in \cB_i}  \norm{\bx_h - \bx_h^i}_1 \right] \\
        &\leq \sum_{h = 1}^H \ee_\nu\left[ \sup_{\theta \in \cB_i} L \cdot \norm{\theta - \theta^i} + 2 D \cdot \sum_{h' = 1}^{h-1} \I\left[ \cA_{h'} \right] \right] \\
        &\leq L H \epsilontil + 2 D H^2 \cdot \max_{h \in [H]}\ee_\nu\left[ \sup_{\theta \in \cB_i } \I\left[ \cA_{h} \right] \right].
    \end{align}
    Now we reason in a similar manner as in Lemma \ref{lem:planningpseudoisometry}:
    \begin{align}
        \ee_\nu\left[ \sup_{\theta \in \cB_i} \I\left[ \cA_{h} \right]\right] &= \pp\left( \bigcup_{\bbaru_{1:H} \in \cB_i} \left\{k_{h'} = k_{h'}^i \text{ for } h' < h \text{ and } k_h \neq k_{h}^i  \right\} \right) \\
        &\leq \sum_{k \neq k' \in [K]} \pp\left( \bigcup_{\theta \in \cB_i} \left\{k_{h'} = k_{h'}^i \text{ for } h' < h \text{ and } \abs{\inprod{\bwstar{k,h} - \bwstar{k',h}}{(\bv_h^i, 1)}} \leq 2 \cdot \norm{\bv_h - \bv_h^i}_1  \right\} \right),
    \end{align}
    where the inequality follows from the same chain of logic as in \eqref{eq:planningunion} and \eqref{eq:planningunion}.  As in \eqref{eq:planningzupper}, we observe that if the mode sequence $k_{h'} = k_{h'}^i$ for $h' < h$, then
    \begin{align}
        \norm{\bv_h - \bv_h^i}_1 \leq 2L \cdot \norm{\theta - \theta^i}_1.
    \end{align}
    Thus, by construction of $\cB_i$, we see that
    \begin{align}
        \pp&\left( \bigcup_{\theta \in \cB_i} \left\{k_{h'} = k_{h'}^i \text{ for } h' < h \text{ and } \abs{\inprod{\bwstar{k,h} - \bwstar{k',h}}{(\bv_h^i, 1)}} \leq 2 \cdot \norm{\bv_h - \bv_h^i}_1  \right\} \right) \\
        &\leq \pp\left(\bigcup_{\theta \in \cB_i} \left\{ \abs{\inprod{\bwstar{k,h} - \bwstar{k',h}}{(\bv_h^i, 1)}} \leq 6 L\cdot \norm{\theta - \theta^i}_1 \right\} \right) \\
        &\leq \pp\left( \abs{\inprod{\bwstar{k,h} - \bwstar{k',h}}{(\bv_h^i,1)}} \leq 6 L \epsilontil \right) \\
        &\leq \frac{6L\epsilontil}{\gamma \sigdir},
    \end{align}
    where the last inequality follows from directional smoothness and the margin assumption, as in \eqref{eq:planningmargin}.  Plugging back in to the definition of $\epsilontil$, we see that $\left\{ (\theta^i, \cB_i) \right\}$ forms a generalized $\epsilon$-bracket as desired.  Note that $\cK$ lives in the $\ell_1$ ball of radius $D$ inside of $\rr^{dH}$; thus, applying the same argument as in \Cref{thm:pwabracketing} concludes the proof.
\end{proof}
Finally, we are ready to prove the main result:
\begin{proof}[Proof of \Cref{thm:planningformal}]
    By \Cref{thm:lazyftplexponential}, it suffices to show that the loss is Lipschitz with respect to $\rho$.  To do this, note that
    \begin{align}
        \ell(\bz_t) - \ell(\bz_t') &\leq \norm{\bv_{t,1:H} - \bv_{t,1:H}'}_1 \\
        &\leq \norm{\theta - \theta'}_1 + \sum_{h = 1}^H \norm{\bx_{t,h} - \bx_{t,h}'}_1 \\
        &= \rho(\theta, \theta').
    \end{align}
    Thus by \Cref{thm:lazyftplexponential}, it holds that if the learner plays Algorithm \ref{alg:lazyftplexponential} with $n = \sqrt{\eta}$, then he experiences
    \begin{align}
        \ee\left[ \reg_T \right] \leq \BigOhTil{\eta + \frac{T D H^2 K^2 L}{\gamma \sigdir} \cdot \sqrt{\frac{Hd}{\eta}}}
    \end{align}
    by appealing to Lemmas \ref{lem:planningpseudoisometry} and \ref{lem:planningbracketing}.  Setting $\eta$ as in the statement of the theorem concludes the proof.
\end{proof}

Finally, we prove an analogue of \Cref{thm:planningformal} where we now assume that the decision boundaries between modes are polynomials.  The statement is almost equivalent to that of \Cref{thm:planningformal}, with the exception that the boundaries between modes are now parameterized by polynomials, with the resulting increased oracle complexity along the lines of \Cref{thm:polyregret}.  The statement is as follows:
\begin{theorem}\label{thm:planningformal_polynomial_boundary}
    Suppose that we are in the situation of \eqref{eq:dynamics}, with the exception that the regions are defined by polynomials of degree at most $r$.  More precisely, we suppose that
    \begin{align}
        \bx_{t,h+1}(\theta ) &= g_{t,h,k_{t,h}(\bv_{t,h}(\theta))}(\bv_{t,h}(\theta)) + \boldeta_{t,h}, \quad \text{and }  \\
        \bu_{t,h}(\theta) &= \bbaru_{t,h} + \bxi_{t,h},  \quad \bv_{t,h}(\theta) = (\bx_{t,h}(\theta), \quad \bu_{t,h}(\theta)), \\
    k_{t,h}(\bv) &= \argmax_{k \in [K]} \phi_{t,h}(k, \bv),  \quad \text{and}\quad \phi_{t,h}(k, \bv) = f_{\bw_{t,k,h}}(\bv),
    \end{align}
    where the $f_{\bw_{t,k,h}}$ are degree $r$ polynomials with $\bw_{t,k,h}$ parameterizing the coefficients.  Suppose that for all $t \in [T]$, the following properties hold almost surely under the adversary's strategy $p_t$:
    \begin{enumerate}
        \item $\bx_{t,1} \mid \cF_t$ and $(\boldeta_{t,h},\bxi_{t,h}) \mid \cF_{t,h-1}$ are $\sigpoly{r}$-polynomially smooth.
        \item For all mode sequences $k_{1:H} \in [K]^{H}$ and $\theta,\theta' \in \cU^{\times H}$, $\norm{G_t(\theta, k_{1:H})-G_t(\theta', k_{1:H})} \leq L\norm{\theta-\theta'}_1$, i.e., the functions $G_t$ are $L$-Lipschitz with respect to the $\ell_1$ norm, where the $G_t$ are the maps defined in Definition \ref{defn:io_map}.
        \item For all $h \in [H]$, $\sup_{\theta \in \cK} \norm{\bu_{t,h}(\theta)}_1 \vee \norm{\bx_{t,h}(\theta)}_1 \leq D$.
        \item For some $\gamma > 0$, it holds for all $h \in [H]$ that $\min_{k\neq k' \in [K]} \coeff_r\left( f_{\bw_{t,k,h}} - f_{\bw_{t,k',h}} \right) \geq \gamma$, where $\coeff_r(\cdot)$ is as defined in Definition \ref{def:polysmooth}.
        \item For all $\bv_{1:H},\bv_{1:H}' \in \cV^H$ with $\|\bv_{1:H}\|_1 \vee \|\bv_{1:H}'\|_1 \le 2D$, we have that the loss functions $\lv_{t}$ are Lipschitz with respect to the $\ell_1$ norm and bounded, i.e.,  $|\lv_{t}(\bv_{1:H}) - \lv_t(\bv_{1:H}')| \le \|\bv_{1:H}-\bv_{1:H}\|_1$ and $|\lv_t(\bv_{1:H})| \le 1$.
        \item The coefficients $\bw_{t,k,h}$ of the polynomials $f_{\bw_{t,k,h}}$ have unit Euclidean norm.
    \end{enumerate}
    If the planner plays $\theta_t$ according to Algorithm \ref{alg:lazyftplexponential} with $\eta = \left( \frac{L T K^2 r^2 H^{2 + r}d^r D B}{\gamma \sigpoly{r}} \right)^{\frac{4r - 2}{4r - 1}}$ and $n = \sqrt{\eta}$, then
    \begin{align}
        \ee\left[ \reg_T \right] \leq \BigOhTil{\left( \frac{L T K^2 r^2 H^{2 + r}d^r D B}{\gamma \sigpoly{r}} \right)^{\frac{4r - 2}{4r - 1}}}.
    \end{align}
    Thus, the oracle complexity of achieving average regret $\epsilon$ is $\BigOhTil{\left( \frac{L K^2 r^2 H^{2 + r}d^r D B}{\gamma \sigpoly{r}} \right)^{\frac{4r - 2}{4r - 1}}} \cdot \epsilon^{- \frac 2r}$.
\end{theorem}
\begin{proof}
    The proof will be similar to that of \Cref{thm:planningformal}.  In particular, we will still use the same $\rho$ as in \eqref{eq:rhoplanning} and preserve notation from that proof.  Applying \Cref{thm:lazyftplexponential}, it suffices to control the pseudo-isometry and generalized bracketing numbers.  We first claim that if $\cM$ is the class of distributions for the adversary, induced by the setting at hand, then
    \begin{align}\label{eq:polyplanningpseudoisometry}
        \sup_{\nu \in \cM} \ee_\nu\left[ \rho(\theta, \theta') \right] \leq \frac{4 D^2 H^2 K^2 L}{\gamma^{\frac 1r} \cdot \sigpoly{r}} \cdot \norm{\theta - \theta'}_1^{\frac 1r}.
    \end{align}
    To see this, we observe that if $\cA_h$ is as in Lemma \ref{lem:xdistupperbound}, then by that same result,
    \begin{align}
        \sup_{\nu \in \cM} \ee_\nu\left[ \rho(\theta, \theta') \right] &= \sup_{\nu \in \cM} \ee_\nu\left[ \norm{\theta - \theta'}  + \sum_{h = 1}^H \norm{\bx_h - \bx_h'}\right] \\
        &\leq 2 L H \cdot \norm{\theta - \theta'}_1 + 2 D H \cdot \sup_{\nu \in \cM} \sum_{h = 1}^H \pp_\nu\left( \cA_h \right).
    \end{align}
    We now compute, as in Lemma \ref{lem:planningpseudoisometry},
    \begin{align}
        \pp\left( \cA_h \right) &= \pp\left( k_{h'} = k_{h'}' \text{ for } h' < h \text{ and } k_h \neq k_h' \right) \\
        &\leq \sum_{k \neq k' \in [K]} \pp\left(k_{h'} = k_{h'}' \text{ for } h' < h \text{ and } f_{\bw_{k,h}}(\bv_h) - f_{\bw_{k',h}}(\bv_h) \geq 0 >  f_{\bw_{k,h}}(\bv_h') - f_{\bw_{k',h}}(\bv_h') \right).
    \end{align}
    For fixed $k \neq k'$, we then compute
    \begin{align}
        \pp&\left(k_{h'} = k_{h'}' \text{ for } h' < h \text{ and } f_{\bw_{k,h}}(\bv_h) - f_{\bw_{k',h}}(\bv_h) \geq 0 >  f_{\bw_{k,h}}(\bv_h') - f_{\bw_{k',h}}(\bv_h') \right) \\
        &\leq \pp\left(k_{h'} = k_{h'}' \text{ for } h' < h \text{ and } \abs{f_{\bw_{k,h}}(\bv_h) - f_{\bw_{k',h}}(\bv_h)} \leq  \abs{f_{\bw_{k,h}}(\bv_h') - f_{\bw_{k',h}}(\bv_h')} + \abs{f_{\bw_{k,h}}(\bv_h) - f_{\bw_{k',h}}(\bv_h)} \right) \\
        &\leq \pp\left( k_{h'} = k_{h'}' \text{ for } h' < h \text{ and } \abs{ f_{\bw_{k,h}}(\bv_h) - f_{\bw_{k',h}}(\bv_h)} \leq 2 D^r \cdot \norm{\bv_h - \bv_h '}_1 \right),
    \end{align}
    where the second inequality follows from the assumption that $\bw_{k,h}$ have unit norm, the fact that power functions area locally Lipschitz, and the fact that $\bv_h$ has norm bounded by $D$.  Applying the assumption of polynomial smoothness, we see that this last expression is bounded above:
    \begin{align}
        \pp&\left( k_{h'} = k_{h'}' \text{ for } h' < h \text{ and } \abs{ f_{\bw_{k,h}}(\bv_h) - f_{\bw_{k',h}}(\bv_h)} \leq 2 D^r \cdot \norm{\bv_h - \bv_h '}_1 \right) \\
        &\leq \frac{2 D}{\gamma^{\frac 1r}\cdot \sigpoly{r}} \norm{\bv_h - \bv_{h'}}_1^{\frac 1r}.
    \end{align}
    We now conclude in a similar manner as in Lemma \ref{lem:planningpseudoisometry} and observe that
    \begin{align}
        \norm{\bv_h - \bv_h'}_1^{\frac 1r} \leq 2 \norm{\bu_h - \bu_h'}_1^{\frac 1r} + 2\norm{\bx_h - \bx_h'}_1^{\frac 1r}
    \end{align}
    and because the mode sequences $k_{1:h-1} = k_{1:h-1}'$ on this event,
    \begin{align}
        \norm{\bx_h - \bx_h'}_1^{\frac 1r} \leq L^{\frac 1r} \cdot \norm{\theta - \theta'}_{1}^{\frac 1r}.
    \end{align}
    Putting everything together suffices to prove that \eqref{eq:polyplanningpseudoisometry} holds.

    We now claim that for $0 < \epsilon < 1$,
    \begin{align}\label{eq:polyplanningbracketing}
        \bracknum\left( \Theta, \cM, \epsilon \right) \leq \left( \frac{12 D^2 H^2 K^2 L}{\sigpoly{r} \gamma^{\frac 1r}} \epsilon \right)^{K^2 r^2 (H d)^r}.
    \end{align}
    To see this, we let $\cN = \left\{ \theta^i \right\} \subset \cU^{\times H}$ denote an $\epsilontil$-net with respect to $\ell_1$, where $\epsilontil = \gamma \left( \frac{\sigpoly{r}}{4 D^2 H^2 K^2 L} \cdot \epsilon \right)^{r}$.  Our proof proceeds similarly to that in Lemma \ref{lem:planningbracketing} and we demonstrate that if $\cB_i$ denots the set of $\epsilontil$-balls around $\theta^i$ then the associated $(\theta^i, \cB_i)$ forms a generalized $\epsilon$-bracket with respect to $\cM$.  Indeed, for fixed $i$ and $\nu \in \cM$, we have:
    \begin{align}
        \ee_\nu\left[ \sup_{\theta \in \cB_i} \rho(\theta, \theta^i) \right] &\leq \epsilontil + \ee_\nu\left[ \sup_{\theta \in \cB_i} \sum_{h = 1}^H \norm{\bx_h - \bx_h^i}_1 \right].
    \end{align}
    For the second term, we again invoke Lemma \ref{lem:xdistupperbound} and observe that
    \begin{align}
        \ee_\nu\left[ \sup_{\theta \in \cB_i} \sum_{h = 1}^H \norm{\bx_h - \bx_h^i}_1 \right] &\leq \sum_{h = 1}^H \ee_\nu\left[ \sup_{\theta \in \cB_i} L \cdot \norm{\theta - \theta^i}_1 + 2 D \cdot \sum_{h' = 1}^{h-1} \I\left[ \cA_{h'} \right] \right] \\
        &\leq L H \epsilontil + 2 D H^2 \cdot \max_{h \in [H]} \ee_\nu\left[ \sup_{\theta \in \cB_i}\I\left[ \cA_h \right]\right].
    \end{align}
    Using the identical logic combining as in Lemma \ref{lem:planningbracketing}, but using the polynomial smoothness assumptions in the same way as in Lemma \ref{lem:polybracketing}, we see that
    \begin{align}
        \pp_\nu\left( \cA_h \right) &\leq \sum_{k \neq k' \in [K]} \pp\left( \bigcup_{\theta \in \cB_i} \left\{k_{h'} = k_{h'}^i \text{ for } h' < h \text{ and } \abs{f_{\bw_{k,h}}(\bv_h^i) - f_{\bw_{k',h}}(\bv_h^i)} \leq 2 D^r \cdot \norm{\bv_h - \bv_h^i}_1  \right\}  \right).
    \end{align}
    Observing once again that
    \begin{align}
        \norm{\bv_h - \bv_h^i}_1 \leq 2 L \cdot \norm{\theta - \theta^i}
    \end{align}
    by the Lipschitzness of the $G_t$, we see that
    \begin{align}
        \pp_\nu\left( \cA_h \right) &\leq \sum_{k \neq k' \in [K]} \pp\left( \bigcup_{\theta \in \cB_i} \left\{k_{h'} = k_{h'}^i \text{ for } h' < h \text{ and } \abs{f_{\bw_{k,h}}(\bv_h^i) - f_{\bw_{k',h}}(\bv_h^i)} \leq 4 D^r L \epsilontil  \right\}  \right) \\
        &\leq \sum_{k \neq k' \in [K]} \pp\left( \abs{f_{\bw_{k,h}}(\bv_h^i) - f_{\bw_{k',h}}(\bv_h^i)} \leq 4 D^r L \epsilontil \right) \\
        &\leq K^2 \frac{4 D L^{\frac 1r}}{\gamma^{\frac 1r} \cdot \sigpoly{r}} \cdot \epsilontil^{\frac 1r}.
    \end{align}
    Plugging everything back in to the above work, we see that indeed $\cN$ is a generalized $\epsilon$-net.  Using the same volume argument as in the proofs of Lemmas \ref{lem:pwabracketing}, \ref{lem:polybracketing}, and \ref{lem:planningbracketing}, we see that \eqref{eq:polyplanningbracketing} holds.

    Finally, we note that $\lv$ is clearly lipschitz with respect to $\rho$ and thus we may apply \Cref{thm:lazyftplexponential}, which tells us that if we run Algorithm \ref{alg:lazyftplexponential}, then for $n = \sqrt{\eta}$,
    \begin{align}
        \ee\left[ \reg_T \right] \leq \BigOhTil{\eta + \frac Tn K^2 r^2 H^r d^r \log\left( \frac 1{\sigpoly{r}} \right) + \frac{TD}{\gamma^{\frac 1r} \sigpoly{r}} \cdot \left( \frac{K^2 r^2 H^r d^r}{\eta} \right)^{\frac 1{4r - 2}}}.
    \end{align}
    Setting $\eta$ as in the statement then concludes the proof.
\end{proof}

\end{document}